\documentclass{article}

\def\shownotes{1}
\usepackage[final]{neurips_2021}

\usepackage[utf8]{inputenc} %
\usepackage[T1]{fontenc}    %
\usepackage{url}            %
\usepackage{amsfonts,bm}       %
\usepackage{nicefrac}       %
\usepackage{microtype}      %
\usepackage{macro_math}
\usepackage{graphicx, subcaption}
\usepackage{multirow}
\usepackage{algorithm}
\usepackage{algorithmic}
\usepackage{natbib}
\usepackage{wrapfig}
\usepackage{enumitem}

\hypersetup{
    colorlinks=true,
    linkcolor=blue,
    filecolor=magenta,      
    urlcolor=blue,
    pdftitle={Overleaf Example},
    pdfpagemode=FullScreen,
    }

\setcitestyle{authoryear,open={(},close={)}}
\renewcommand{\cite}{\citep}

\title{Improved Regularization and Robustness for Fine-tuning in Neural Networks}

\author{Dongyue Li\\
Northeastern University, Boston\\
\texttt{li.dongyu@northeastern.edu}
\And
Hongyang R. Zhang\\
Northeastern University, Boston\\
\texttt{ho.zhang@northeastern.edu}
}

\begin{document}
	\maketitle
    \begin{abstract}
A widely used algorithm for transfer learning is fine-tuning, where a pre-trained model is fine-tuned on a target task with a small amount of labeled data. When the capacity of the pre-trained model is significantly larger than the size of the target dataset, fine-tuning is prone to overfitting and memorizing the training labels. Hence, a crucial question is to regularize fine-tuning and ensure its robustness against noise. To address this question, we begin by analyzing the generalization properties of fine-tuning. We present a PAC-Bayes generalization bound that depends on the \textit{distance traveled in each layer} during fine-tuning and the \textit{noise stability} of the fine-tuned model. We empirically measure these quantities. Based on the analysis, we propose \textit{regularized self-labeling}---the interpolation between regularization and self-labeling methods, including (i) \textit{layer-wise regularization} to constrain the distance traveled in each layer; (ii) \textit{self-label-correction and label-reweighting} to correct mislabeled data points (that the model is confident) and reweight less confident data points. We validate our approach on an extensive collection of image and text datasets using multiple pre-trained model architectures. Our approach improves  baseline methods by 1.76\% (on average) for seven image classification tasks and 0.75\% for a few-shot classification task. When the target data set includes noisy labels, our approach outperforms baseline methods by an average of 3.56\% in two noisy settings.
\end{abstract}
    \section{Introduction}

Learning from limited or weakly labeled data is a fundamental problem in many real-world applications \cite{ratner2016data,ratner2017snorkel}.
A common approach to address this problem is fine-tuning a large model that has been pre-trained on publicly available labeled data.
Since fine-tuning is typically applied to a target task with limited labels, this procedure is prone to overfitting or memorization in practice.
These issues become worse when the target task contains noisy labels \cite{zhang2016understanding}.
In this paper, we analyze regularization methods for fine-tuning from both theoretical and empirical perspectives.
Based on the analysis, we propose a \textit{regularized self-labeling} approach that improves the generalization and robustness properties of fine-tuning.

Previous works~\cite{li2018delta,li2018explicit} have proposed regularization methods to constrain the distance between a fine-tuned model and the pre-trained model in the Euclidean norm.
\citet{li2020rethinking} provides an extensive study to show that the performance of fine-tuning and the benefit of adding regularization depend on the hyperparameter choices.
\citet{salman2020adversarially} empirically find that performing adversarial training during the pre-training phase helps learn pre-trained models that transfer better to downstream tasks.
The work of \citet{gouk2020distance} generalizes the above ideas to various norm choices and finds that projected gradient descent methods perform well for implementing distance-based regularization.
Additionally, they derive generalization bounds for fine-tuning using Rademacher complexity. 
These works focus on settings where there is no label noise in the target dataset.
When label noise is present, for example, due to the application of weak supervision techniques~\cite{ratner2016data}, a crucial question is to design methods that are robust to such noise.
The problem of learning from noisy labels has a rich history of study in supervised learning \cite{natarajan2013learning}. In contrast, little is known in the transfer learning setting.
These considerations motivate us to analyze the generalization and robustness properties of fine-tuning.

In Section \ref{sec_reg}, we begin by conducting a PAC-Bayesian analysis of regularized fine-tuning.
This is inspired by recent works that have found PAC-Bayesian analysis correlates with empirical performance better than Rademacher complexity \cite{jiang2019fantastic}.
We identify two critical measures for analyzing the generalization performance of fine-tuning.
The first measure is the $\ell_2$ norm of the distance between the pre-trained model (initialization) and the fine-tuned model.
The second measure is the perturbed loss of the fine-tuned model, i.e., its loss after the model weights get perturbed by random noise.
First, we observe that the fine-tuned weights remain close to the pre-trained model. 
Moreover, the top layers travel much further away from the pre-trained model than the bottom layers.
Second, we find that fine-tuning from a pre-trained model implies better noise stability than training from a randomly initialized model. 
In Section \ref{sec_robust}, we evaluate regularized fine-tuning for target tasks with noisy labels.
We find that fine-tuning is prone to memorizing the noisy labels, and regularization helps alleviate such memorization behavior. 
Moreover, we observe that the neural network has not yet overfitted to the noisy labels during the early phase of fine-tuning. Thus, its prediction could be used to relabel the noisy labels.

\begin{wrapfigure}[15]{r}{0.5\textwidth}
    \vspace{-0.75cm}
    \begin{center}
        \includegraphics[width=0.95\linewidth]{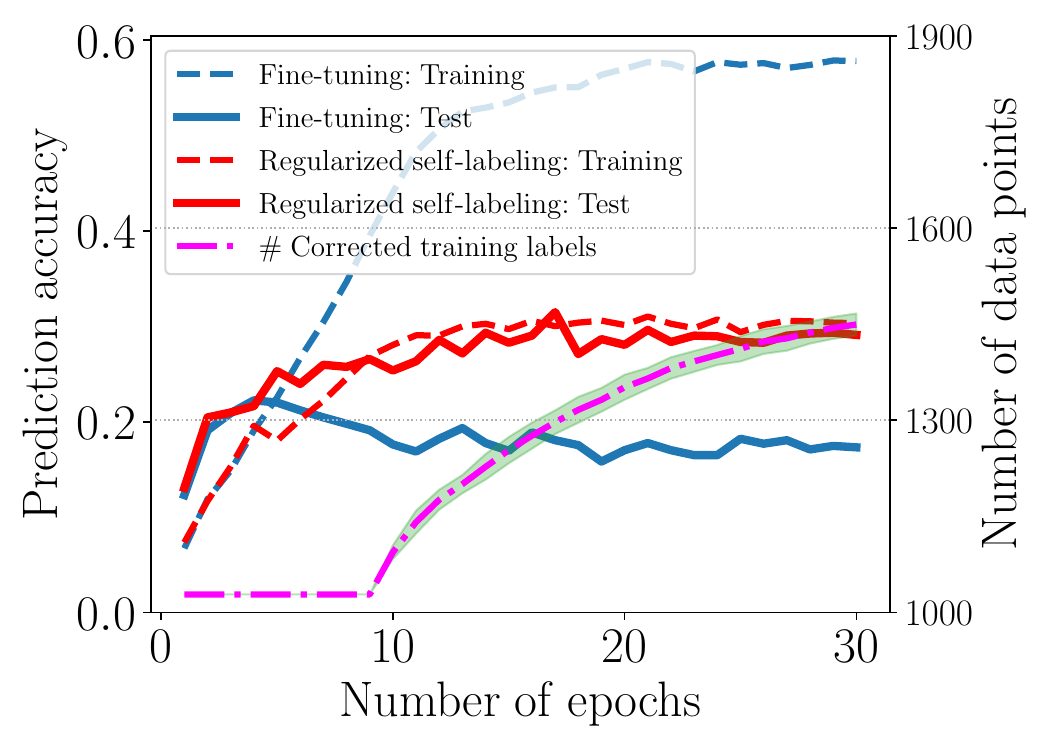}
    \end{center}
    \vspace{-0.5cm}
    \caption{\textcolor{red}{Red:} Layer-wise regularization closes generalization gap. \textcolor{magenta}{Magenta:} Self-labeling relabels noisy data points to their correct label.}
    \label{fig:main_figure}
\end{wrapfigure}
We propose an algorithm that incorporates layer-wise regularization and self-labeling to enhance regularization and robustness, as supported by our results.
Figure \ref{fig:main_figure} illustrates the two components.
First, we encode layer-wise distance constraints to regularize the model weights at different levels. Compared to (vanilla) fine-tuning, our algorithm reduces the gap between the training and test accuracy, thus alleviating overfitting.
Second, we introduce a self-labeling mechanism that corrects and reweights noisy labels based on the neural network's predictions.
Figure \ref{fig:main_figure} shows that our algorithm effectively hinders the model from learning the incorrect labels by relabeling them to correct ones.

In Section \ref{sec_exp}, we evaluate our proposed algorithm for both transfer learning and few-shot classification tasks with image and text data sets.
First, using ResNet-101 \cite{he2016deep} as the pre-trained model, our algorithm outperforms previous fine-tuning methods on seven image classification tasks by an average of $1.76\%$ and $3.56\%$ when their labels are noisy, respectively.
Second, we find qualitatively similar results for applying our approach to medical image classification tasks (ChestX-ray14 \cite{wang2017chestx, rajpurkar2017chexnet}) and vision transformers \cite{dosovitskiy2020image}.
Finally, we extend our approach to few-shot learning and sentence classification. For these related but different tasks and data modalities,
we find an improvement of $0.75\%$ and $0.46\%$ over previous methods, respectively.

In summary, our contributions are threefold.
First, we provide a PAC-Bayesian analysis of regularized fine-tuning.
Our result implies empirical measures that explain the generalization performance of regularized fine-tuning. 
Second, we present a regularized self-labeling approach to enhance the generalization and robustness properties of fine-tuning.
Third, we validate our approach on an extensive collection of classification tasks and pre-trained model architectures.

    \section{Related Work}

Fine-tuning is widely used in multi-task and transfer learning, meta-learning, and few-shot learning.
Previous works \cite{li2018explicit,li2018delta} find that injecting $\ell_2$ regularization helps improve the empirical performance of fine-tuning.
\citet{li2018explicit} propose a $\ell_2$ distance regularization method that penalizes the $\ell_2$ distance between the fine-tuned weights and the pre-trained weights.
\citet{li2018delta} penalize the distance between the \textit{feature maps} as opposed to the layer weights.
\citet{chen2019catastrophic} design a regularization method that suppresses the spectral components (of the \textit{feature maps}) with small singular values to avoid negative transfer.
\citet{gouk2020distance} instead encode distance constraints in constrained minimization and uses projected gradient descent to ensure the weights are close to the pre-trained model. \citet{salman2020adversarially} show that fine-tuning from adversarially robust pre-trained models outperforms fine-tuning from (standard) pre-trained models.

The robustness of learning algorithms in the presence of label noise has been extensively studied in supervised learning~\cite{natarajan2013learning}.
Three broad ideas for designing robust algorithms include defining novel losses, identifying noisy labels, and using regularization methods. \citet{zhang2018generalized} design the robust Generalized Cross Entropy~(GCE) loss which is a mixture of Cross Entropy~(CE) and mean absolute error. The Symmetric Cross Entropy~(SCE)~\cite{wang2019symmetric} loss combines reverse cross-entropy with the CE loss. \citet{ma2020normalized} proposes normalizing the loss to be robust to noisy labels and combines active and passive loss functions.
\citet{thulasidasan2019combating} propose DAC, which identifies and suppresses the signals of noisy samples by abstention-based training. 
\citet{liu2020early} introduce an early learning regularization approach to mitigate label memorization.
\citet{huang2020self} propose a self-adaptive training method that corrects noisy labels and reweights training data to suppress erroneous signals.
These works primarily concern the supervised learning setting.
To the best of our knowledge, fine-tuning algorithms under label noise are relatively under-explored in transfer learning.
Our approach draws inspiration from the semi-supervised learning literature, which has evaluated pseudo-labeling and self-training approaches given a limited amount of labeled data and a large amount of unlabeled data. 
Recent work considers a sharpness-aware approach and evaluates its performance for fine-tuning from noisy labels.
It would be interesting to explore whether combining their approach with our ideas could yield better results.

From a theoretical perspective, the work of \citet{ben2010theory} considers a setting where labeled data from many source tasks and a target task are available.
They demonstrate that minimizing a weighted combination of the source and target empirical risks yields the best result.
In the supervised setting, \citet{arora2019fine} provide data-dependent generalization bounds for neural networks based on the noise resilience of the trained network.
\citet{nagarajan2018deterministic} provide improved generalization bounds that depend only polynomially on the depth of the neural network characterized by a margin condition.
Recent work has found that the PAC-Bayes theory provides generalization measures that correlate more closely with empirical generalization performance than other alternatives \cite{jiang2019fantastic}.
We defer a more extensive review of PAC-Bayesian generalization theory to Section \ref{sec:proof_details}.

	\section{Preliminaries}\label{sec_setup}

\textbf{Problem setup.} We begin by formally introducing the  setup.
Suppose we would like to solve a target task.
We have a training data set of size $n^{(t)}$.
Let $(x^{(t)}_1, y^{(t)}_1), \dots, (x^{(t)}_{n^{(t)}}, y^{(t)}_{n^{(t)}})$ be the feature vectors and the labels of the training data set.
We assume that every $x^{(t)}_i$ lies in a $d$-dimensional space denoted by $\cX\subset\real^{d}$.
Following standard terminologies in statistical learning, we assume all data are drawn from some unknown distribution supported on $\cX\times\cY$, where $\cY \subseteq \real$ is the label space.
Denote the underlying data distribution as $\cP^{(t)}$.

For our result in Section \ref{sec_reg}, we consider feedforward neural networks, and our results can also be extended to convolutional neural networks (see, e.g., \citet{long2020generalization}). %
Consider an $L$ layer neural network.
For each layer $i$ from $1$ to $L$, let $\psi_i$ be the activation function and let $W_i$ be the weight matrix at layer $i$.
Given an input $z$ to the $i$-th layer, the output is then denoted as $\phi_i(z) = \psi_i(W_i z)$.
Thus, the final output of the network is given by
\[ f_W(x) = \phi_L \circ \phi_{L-1} \circ \cdots \phi_1(x), \text{ for any input } x \in \cX, \]
where we use $W = [W_1, \dots, W_L]$ to include all the parameters of the network for ease of notation.
The prediction error of $f_W$ is measured by a loss function $\ell(\cdot)$ that is both convex and $1$-Lipschitz-continuous.
\begin{align}\label{eq_pred}
	\cL^{(t)}(f_W) = \exarg{(x, y) \sim \cP^{(t)}}{\ell(f_W(x), y)}.
\end{align}
Suppose we have access to a pre-trained source model $\hat{W}^{(s)}$.
Using $\hat{W}^{(s)}$ as an initialization, we then fine-tune the layers $\hat{W}^{(s)}_1, \dots, \hat{W}^{(s)}_L$ to predict the target labels.

\textbf{Applications.}
We consider seven image classification data sets described in Table \ref{tab: dataset_statistics}.
We use ResNets pre-trained on ImageNet as the initialization $\hat{W}^{(s)}$~\cite{russakovsky2015imagenet,he2016deep}.
We perform fine-tuning using the pre-trained network on the above data sets.
See Section \ref{sec_exp_setup} for further description of the training procedure.

\begin{table*}[!t]
\centering
\caption{Basic statistics for seven image classification tasks.}\label{tab: dataset_statistics}
\begin{small}
\begin{tabular}{@{}llcccc@{}}
\toprule
Datasets                      & Training & Validation & Test & Classes \\ \midrule
Aircrafts~\cite{maji2013fine}     & 3334  & 3333       & 3333 & 100     \\
CUB-200-2011~\cite{WahCUB_200_2011} & 5395  & 599        & 5794 & 200     \\
Caltech-256~\cite{griffin2007caltech}      & 7680  & 5120       & 5120 & 256     \\
Stanford-Cars~\cite{stanford_cars2013}   & 7330  & 814        & 8441 & 196     \\
Stanford-Dogs~\cite{stanford_dogs2011}   & 10800 & 1200       & 8580 & 120     \\
Flowers~\cite{nilsback2008automated}   & 1020  & 1020       & 6149 & 102     \\
MIT-Indoor~\cite{sharif2014cnn}        & 4824  & 536        & 1340 & 67      \\ \bottomrule
\end{tabular}
\end{small}
\vspace{-0.1in}
\end{table*}

\textbf{Modeling label noise.}
In many settings, fine-tuning is applied to a target task whose labels may contain noise; for example, if the target labels are created using weak supervision techniques \cite{ratner2019training,saab2021observational}.
To capture such settings, we denote a noisy data set as $(x^{(t)}_1, \tilde{y}^{(t)}_1), \dots, (x^{(t)}_{n^{(t)}}, \tilde{y}^{(t)}_{n^{(t)}})$, where $\tilde{y}_{i}^{(t)}$ is a noisy version of $y_{i}^{(t)}$.
We consider two types of label noise: \textit{independent} noise and \textit{correlated} noise.
We say that the label noise is \textit{independent} if it is independent of the input feature vector
\begin{equation*}
	\Pr(\tilde{y}_{i}^{(t)} = k | {y}_{i}^{(t)} = j, x_{i}^{(t)}) = \Pr(\tilde{y}_{i}^{(t)} = k | {y}_{i}^{(t)} = j) = \eta_{j,k}
\end{equation*}
for some fixed $\eta_{j,k}$ between $0$ and $1$.
On the other hand, we say that the label noise is \textit{correlated} if it depends on the input feature vector (i.e. the above equation does not hold).
	\section{Our Proposed Approaches}\label{sec_theory}

Given the problem setup described above, next, we study the generalization and robustness properties of regularized fine-tuning methods.
First, we present a PAC-Bayes generalization bound for regularized fine-tuning. This result motivates us to evaluate two empirical measures: the fine-tuned distance in each layer and the perturbed loss of the fine-tuned model.
Second, we consider fine-tuning from noisy labels. We demonstrate that layer-wise regularization prevents the model from memorizing noisy labels.
We then suggest injecting predictions of the model during training, similar to self-training and pseudo-labeling.
Finally, we incorporate both components into our \textit{regularized self-labeling} approach, blending the strengths of both to improve the generalization performance and the robustness of fine-tuning.

\subsection{Fine-tuning and regularization}\label{sec_reg}

We consider the following regularized fine-tuning problem, which constrains the network from traveling too far from the pre-trained initialization $\hat{W}^{(s)}$ \cite{li2018delta,li2018explicit,gouk2020distance}.
\begin{align}
	\hat{W} \leftarrow\arg\min&~ \hat{\cL}^{(t)}(f_W) \label{eq_reg}\\
	\mbox{ s.t. } & \normFro{W_i - \hat{W}_i^{(s)}} \le D_i,~\forall\, i = 1,\dots,L. \label{eq_constraint}
\end{align}

 \begin{wrapfigure}[13]{r}{0.46\textwidth}
 \vspace{-0.3in}
   \begin{center}
   \includegraphics[width=0.9\linewidth]{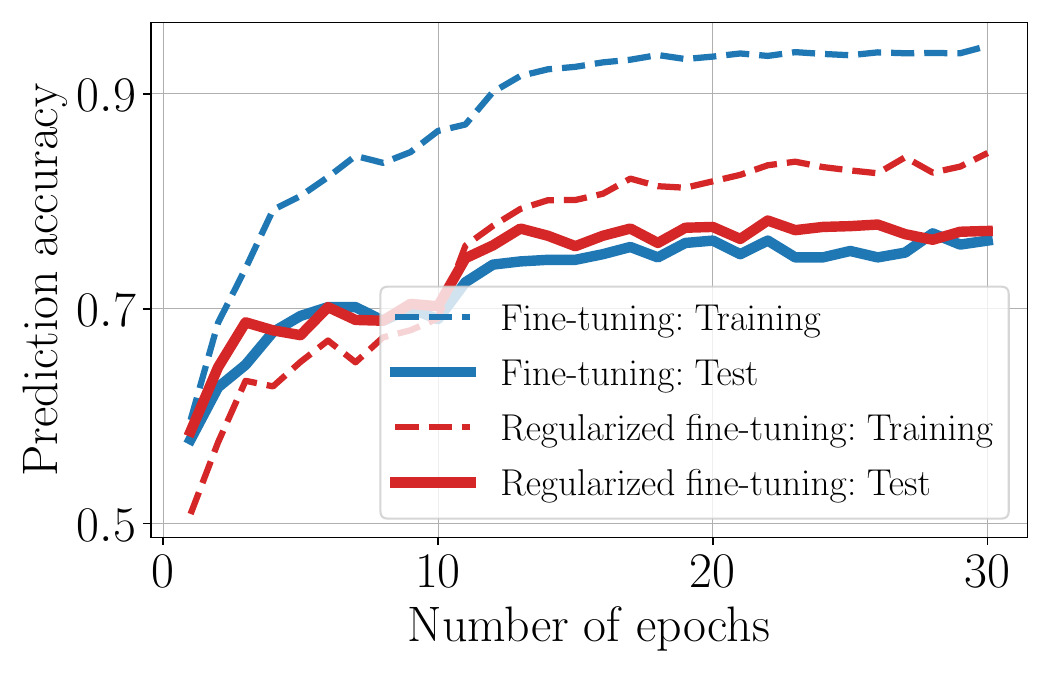}
   \end{center}\vspace{-0.13in}
   \caption{The training and test accuracy of fine-tuning using early stopping vs. optimizing \eqref{eq_reg} and \eqref{eq_constraint} on the Indoor data set.}
 \label{fig:finetune_no_noise}
 \end{wrapfigure}
Above, $D_i$ is a hyperparameter that constrains how far the $i$-th layer $W_i$ can travel from the pre-trained initialization $\hat{W}_i^{(s)}$.
Previous works have observed that stronger regularization reduces the generalization gap during fine-tuning (cf. Figure \ref{fig:finetune_no_noise}). 
Next, we analyze the generalization error of $\hat{W}$, that is, $\cL^{(t)}(f_{\hat W}) - \hat\cL^{(t)}(f_{\hat W})$, where $\cL^{(t)}(f_{\hat W})$ is the test loss of $f_{\hat W}$ according to equation \eqref{eq_pred} and $\hat\cL^{(t)}(f_{\hat W})$ is the empirical loss of $f_{\hat{W}}$ on the training data set.

\noindent\textbf{PAC-Bayesian analysis.} We begin by analyzing the generalization error of $f_{\hat W}$ using PAC-Bayesian \clearpage
tools \cite{mcallester1999pac,mcallester1999some}.
This departs from the previous work of \citet{gouk2020distance}, which is based on their analysis using Rademacher complexity.
This change of perspective is inspired by several recent works that have found PAC-Bayesian bounds to better correlate with empirical performance than Rademacher complexity bounds \cite{jiang2019fantastic}.
We refer the interested reader to Section \ref{sec:proof_details} for further references from this line of work.
After presenting our results, we provide empirical measures of our results and discuss the practical implications.

\begin{theorem}[PAC-Bayes generalization bound for fine-tuning]\label{thm_fnn}
Suppose for every $i = 1,\dots,L$, $\norm{\hat{W}_i^{(s)}}_2 \le B_i$, for a fixed $B_i > 1$.
Suppose the feature vectors in the domain $\cX$ are all bounded: $\norm{x}_2 \le C_1$ for every $x\in\cX$, for some $C_1 \ge 1$.
Finally, suppose the loss function $\ell(\cdot)$ is 1-Lipschitz and bounded from above by a fixed constant $C_2$.
Under these conditions, let $f_{\hat{W}}$ be the minimizer of regularized fine-tuning, solved from problem \eqref{eq_constraint}.
Let $\varepsilon > 0$ be an arbitrary small value.
Let $H$ be the maximum over the width of all the $L$ layers and the input dimension $d$.
Then, with probability at least $1 - 2\delta$ for some small $\delta > 0$, the expected loss $\cL^{(t)}(f_{\hat W})$ is upper bounded by
{%
\begin{align}\label{eq_pac_main}
	\resizebox{.93\hsize}{!}{$\cL^{(t)}(f_{\hat{W}}) \le \hat{\cL}^{(t)}(f_{\hat{W}}) 
	+ \varepsilon + C_2\sqrt{\frac{{{ \frac{36}{\varepsilon^2} \cdot C_1^2 H \log(4LH C_2) \Big(\sum_{i=1}^L \frac{\prod_{j=1}^L (B_j + D_j)}{B_i + D_i} \Big)^2 \Big(\sum_{i=1}^L D_i^2\Big)}  {} + 3\ln{\frac {n^{(t)}} {\delta}}} + 8}{{n^{(t)}}}}.$}
\end{align}}
\end{theorem}

\textit{Relation to prior works.} Compared to the result of \citet{gouk2020distance}, we instead proceed by factoring the generalization error into a noise error (i.e., the $\varepsilon$ term in equation \eqref{eq_pac_main}) and a KL-divergence between $\hat{W}^{(s)}$ and $\hat{W}$ (i.e., the final term in equation \eqref{eq_pac_main}).
The proof of Theorem \ref{thm_fnn} is based on \citet{neyshabur2017pac} and is presented in Appendix \ref{sec:proof_details}.
The difference between Theorem \ref{thm_fnn} and the result of \citet{neyshabur2017pac} is that our result is stated for the  (e.g. cross-entropy) loss function $\ell(\cdot)$ whereas \citet{neyshabur2017pac} states the result using the soft margin loss.

Our result suggests that the \textit{fine-tuned distances} $\set{D_i}_{i=1}^L$ and the \textit{perturbed loss} (more precisely $\exarg{U}{\ell(f_{\hat W + U}(x), y)}$ where every entry of $U$ is drawn from $\cN(0, \sigma^2)$) are two important measures for fine-tuning.
Next, we empirically evaluate these two measures using real-world datasets.

{\small
\begin{figure*}[!t]
    \begin{minipage}[b]{0.48\textwidth}
    	\centering
	    \includegraphics[width=0.95\textwidth]{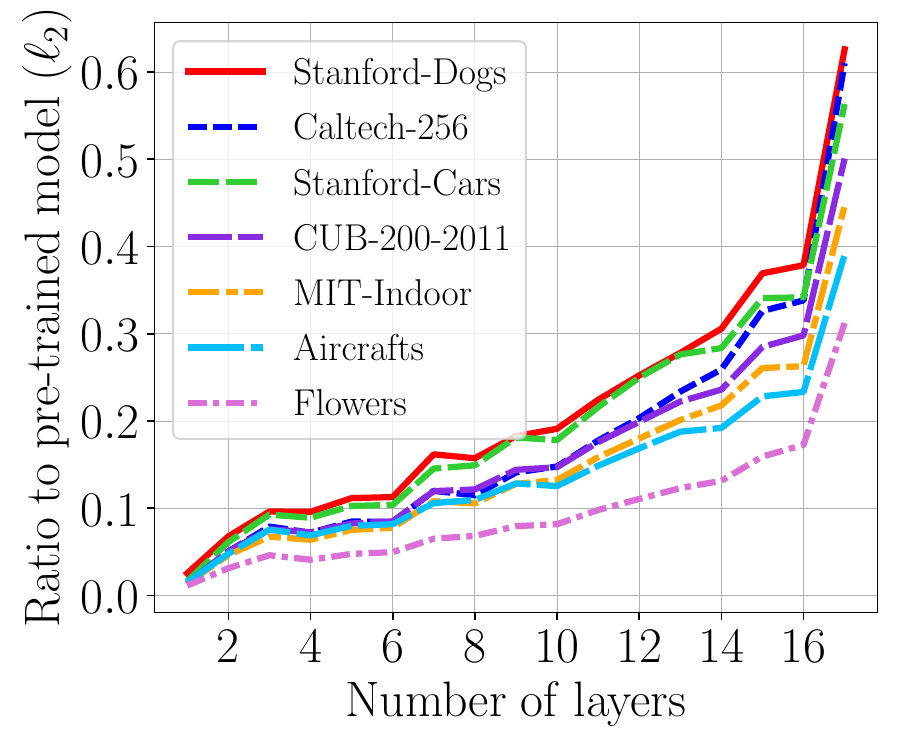}
	    \subcaption{Fine-tuned distances.}
    	\label{fig:layer_wise_dis}
    \end{minipage}\hfill
    \begin{minipage}[b]{0.48\textwidth}
        \centering
        \begin{small}
        \begin{tabular}{@{}ccccc@{}}
        \toprule
              & \multicolumn{3}{c}{CUB-200-2011}  \\ 
        $\sigma$ & Random        & Pre-trained       & Adversarial \\ \midrule
        $10^{-2}$      & 3.77$\pm$0.42 & \textbf{1.45$\pm$0.13} & 1.76$\pm$0.09 \\
        $10^{-3}$      & 0.82$\pm$0.07 & 0.62$\pm$0.03 & \textbf{0.54$\pm$0.03} \\
        $10^{-4}$      & 0.81$\pm$0.04 & 0.61$\pm$0.03 & \textbf{0.61$\pm$0.01} \\ \midrule
              & \multicolumn{3}{c}{Indoor} \\
        $\sigma$ & Random       & Pre-trained       & Adversarial\\ \midrule
        $10^{-2}$      & 2.51$\pm$0.34  & 1.11$\pm$0.09 & \textbf{0.97$\pm$0.07} \\
        $10^{-3}$      & 0.49$\pm$0.09  & 0.36$\pm$0.05 & \textbf{0.32$\pm$0.04} \\
        $10^{-4}$      & 0.44$\pm$0.03  & 0.33$\pm$0.02 & \textbf{0.30$\pm$0.04} \\
              \bottomrule
        \end{tabular}
        \end{small}
        \vspace{0.3in}
        \subcaption{Perturbed loss.}
        \label{fig:stability}
    \end{minipage}
    \caption{Left: Ratio between the $\ell_2$-norm of the fine-tuned distances and the pre-trained network.
    Right: Perturbed loss of fine-tuned model from random initializations (Random), pre-trained model initializations (Pre-trained), adversarially trained model initializations (Adversarial). Results are shown for the CUB-200-2011 and the Indoor data sets using ResNet-18, averaged over {10} runs.}
    \vspace{-0.15in}
\end{figure*}}

\textbf{Fine-tuned distances $\bm{\set{D_i}_{i=1}^L}$.}
First, we present our empirical findings for the fine-tuned distances in each layer. We fine-tune a ResNet-18 model (pre-trained on ImageNet) on seven data sets and calculate the Frobenius distance between $\hat{W}^{(s)}_i$ and $\hat{W}_i$ for every layer $i$. 

Figure~\ref{fig:layer_wise_dis} shows the ratio of the fine-tuned distances to the pre-trained weights in each layer.
First, we observe that the fine-tuned distances are relatively small compared to the pre-trained network.
This means that fine-tuning stays within a small local region near the pre-trained model.
Second, we find that $D_i$ varies across layers. 
$D_i$ is smaller at lower layers and is larger at higher layers.
This observation aligns with the folklore intuition that different layers in Convolutional neural networks play a different role~\cite{guo2019spottune}.
Bottom layers extract higher-level representations of features and stay close to the pre-trained model.
The top layers extract class-specific features and vary among different tasks. Thus, the top layers travel further away from the pre-trained model.

Based on the above empirical findings, we propose layer-wise constraints as regularization. In particular, the constraints for the bottom layers should be small, and the constraints for the top layers should be large.
In practice, we find that an exponentially increasing scheme involving a base distance parameter $D$ and a scale factor $\gamma$ works well.
This is summarized in the following regularized fine-tuning problem:
\begin{align}
	\hat{W} \leftarrow\arg\min&~ \hat{\cL}^{(t)}(f_W) \label{eq_reg_2} \\
	\mbox{ s.t. } & \normFro{W_i - \hat{W}_i^{(s)}} \le D \cdot \gamma^{i-1},~\forall\, i = 1,\dots,L. \label{eq_constraint_lw}
\end{align}
where the scale factor $\gamma > 1$ according to Figure \ref{fig:layer_wise_dis}.
The previous work of~\citet{gouk2020distance} uses the same distance parameter in equation \eqref{eq_constraint_lw}, i.e., $D_i = D,\ \forall i = 1, \ldots, L$.
In our experiments, we have observed that such constant regularization constraints are only active at the top layers.
With layer-wise regularization, we instead extend the constraints to the bottom layers as well.
In Table \ref{tab:distances} (see Section \ref{sec:ablation_study}), we compare the value of $\sum_{i=1}^L D_i^2$ between constant regularization and layer-wise regularization.
We find that layer-wise regularization leads to a smaller value of $\sum_{i=1}^L D_i^2$, thus indicating a tighter generalization bound according to Theorem \ref{thm_fnn}.

\begin{figure*}[!t]
    \begin{subfigure}[b]{0.49\textwidth}
    	\centering
	    \includegraphics[width=0.8\textwidth]{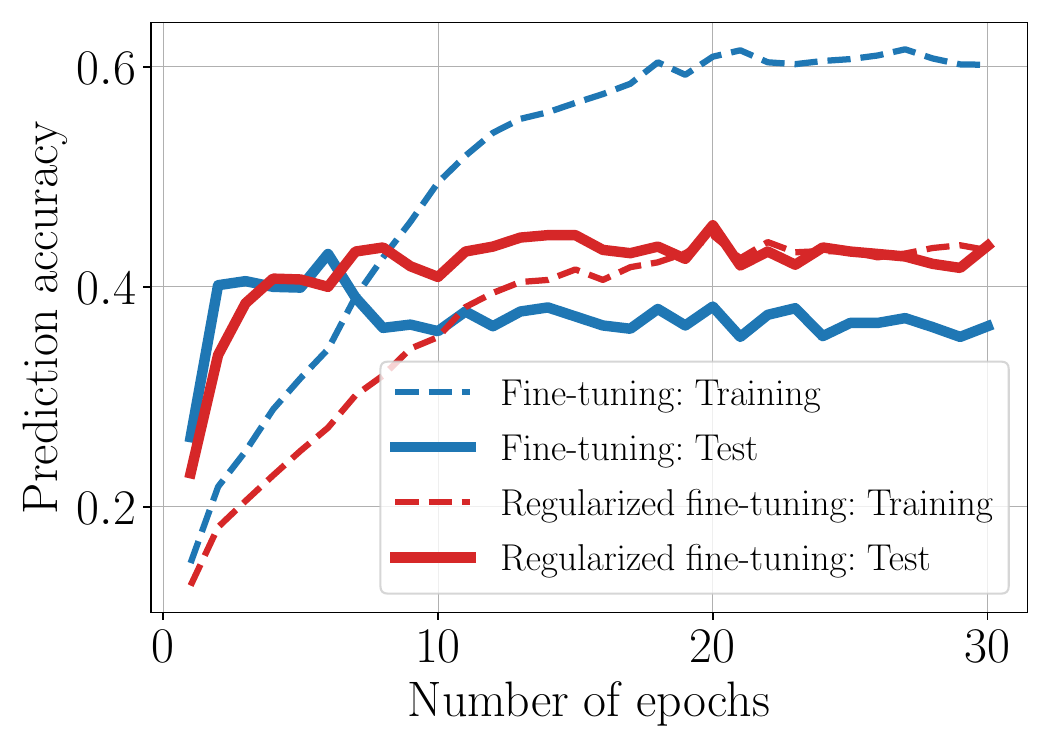}
	    \caption{Noise rate $\eta$ = 60\%.}
    \end{subfigure}
    \begin{subfigure}[b]{0.49\textwidth}
        \centering
        \includegraphics[width=0.8\textwidth]{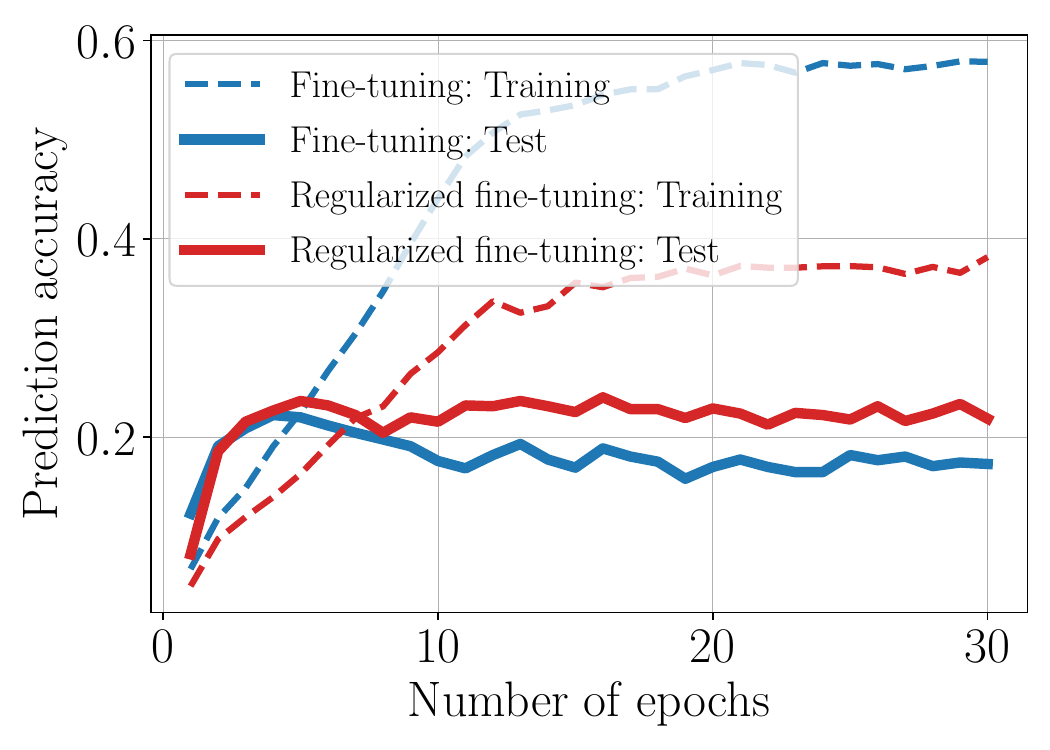}
        \caption{Noise rate $\eta$ = 80\%.}
    \end{subfigure}
    \caption{Training and test accuracy of fine-tuning using early stopping and optimizing equations \eqref{eq_reg_2} and \eqref{eq_constraint_lw} on the Indoor data set with different levels of noise rate. Stronger regularization reduces the generalization gap in both settings. }\label{fig:memorization}
    \vspace{-0.1in}
\end{figure*}

\textbf{Perturbed loss $\bm{\exarg{U}{\ell(f_{\hat W + U}(x), y)}}$.}
Second, we compare the perturbed loss of models fine-tuned from random, pre-trained, and adversarially robust pre-trained model initializations. 
The perturbed loss is measured by the average loss on the training set, after perturbing each entry of the layer weights by a Gaussian random variable $N(0, \sigma^2)$.

In Table \ref{fig:stability}, we report the perturbed losses for both the CUB-200-2011 and the Indoor data set. The results show that models fine-tuned from pre-trained initializations are more stable than models fine-tuned from random initializations.

\textit{Explaining why adversarial robustness helps fine-tuning.}
Recent work \cite{salman2020adversarially} found that performing adversarial training in the pre-training phase leads to models that transfer better to downstream tasks.
More precisely, the adversarial training objective is defined as
\begin{align}\label{eq_adv}
    \tilde{\cL}^{(s)}_{\adv}(\theta) = \min_{W} \frac{1}{n^{(s)}} \sum_{i=1}^{n^{(s)}} \max_{\norm{\delta}_2 \le \varepsilon} \ell(f_W(x_i^{(s)} + \delta), y_i^{(s)}).
\end{align}
Let $\hat{W}^{(s)}_{\adv}$ be a fine-tuned neural network using a pre-trained model from adversarial training.
We measure the perturbed loss of $\hat{W}^{(s)}_{\adv}$ using the pre-trained model provided by \citet{salman2020adversarially} on the CUB-200-2011 and Indoor datasets. Table \ref{fig:stability} shows that $f_{\hat{W}^{(s)}_{\adv}}$ often incurs lower perturbed losses than models fine-tuned from random and pre-trained initializations.
In Section \ref{sec:ablation_study}, we additionally observe that the layers of $\hat{W}^{(s)}_{\adv}$ generally have lower Frobenious norms compared to $\hat{W}^{(s)}$.
Thus, the improved noise stability and the smaller layer-wise norms together imply a tighter generalization bound for adversarial training.
We leave a thorough theoretical analysis of adversarial training in the context of fine-tuning to future works.

\subsection{Fine-tuning and robustness}\label{sec_robust}

Next, we extend our approach to fine-tuning from noisy labels.
We compare fine-tuning with and without regularization on a target task with independent label noise, which motivates our approach.
In Figure \ref{fig:memorization}, we plot the training and test accuracy curves on the Indoor dataset with two noise rates.
We observe that the test accuracy increases initially, indicating that the network learns from the correct labels during the first few epochs. As the learning process progresses, the test accuracy of fine-tuning decreases, indicating that the network is overfitting to noisy labels.
Furthermore, the gap between the training and test accuracy curves exceeds 20\%.
Similar observations have been presented in prior works~\cite{zhang2016understanding,li2018algorithmic,liu2020early}, where over-parametrized neural networks can memorize the entire training set.
One way to mitigate memorization is to reduce model capacity explicitly \cite{WZR20,yang2020analysis}.
Another approach is via explicit regularization, as the distance constraint significantly reduces the training-test accuracy gap and improves test accuracy by a substantial margin.
Based on our empirical observation, the neural network has some discriminating power during the early fine-tuning phase.
Thus, we can leverage the model predictions to relabel incorrect labels and reweight data points with incorrect labels. We describe the two components next.

\noindent\textbf{Self label-correction.}
We propose a label correction step to augment the number of correct labels. We leverage the discriminating power of a network during the early phase and correct data points for which the model has high confidence.
Concretely, let $\bm{p_i}$ be the prediction of the $i$-th data point.
Given a confidence threshold $p_t$ at epoch $t$, if $\max(\bm{p_i}) \geq p_t$ and $\arg \max({\bm{p_i}}) \neq \tilde{y}^{(t)}_i$, we identify data point $i$ as noisy and relabel it to $$\tilde{y}^{(t)}_i = \arg \max({\bm{p_i}}).$$
Figure \ref{fig:label_correction} illustrates the precision of this step.
We define precision as the number of correctly relabeled data points divided by the total number of data points whose labels are relabeled.
We observe that precision is high at the beginning (around $60\%$) and gradually decreases in later epochs. Moreover, we find that precision increases with regularization.
This suggests an intricate interaction between regularization and self-labeling during fine-tuning. We elaborate on the role of regularization in label correction by looking at the number of relabeled data points in Section \ref{sec:ablation_study}.

\begin{figure*}[!t]
    \begin{subfigure}[b]{0.49\textwidth}
    	\centering
	    \includegraphics[width=0.8\textwidth]{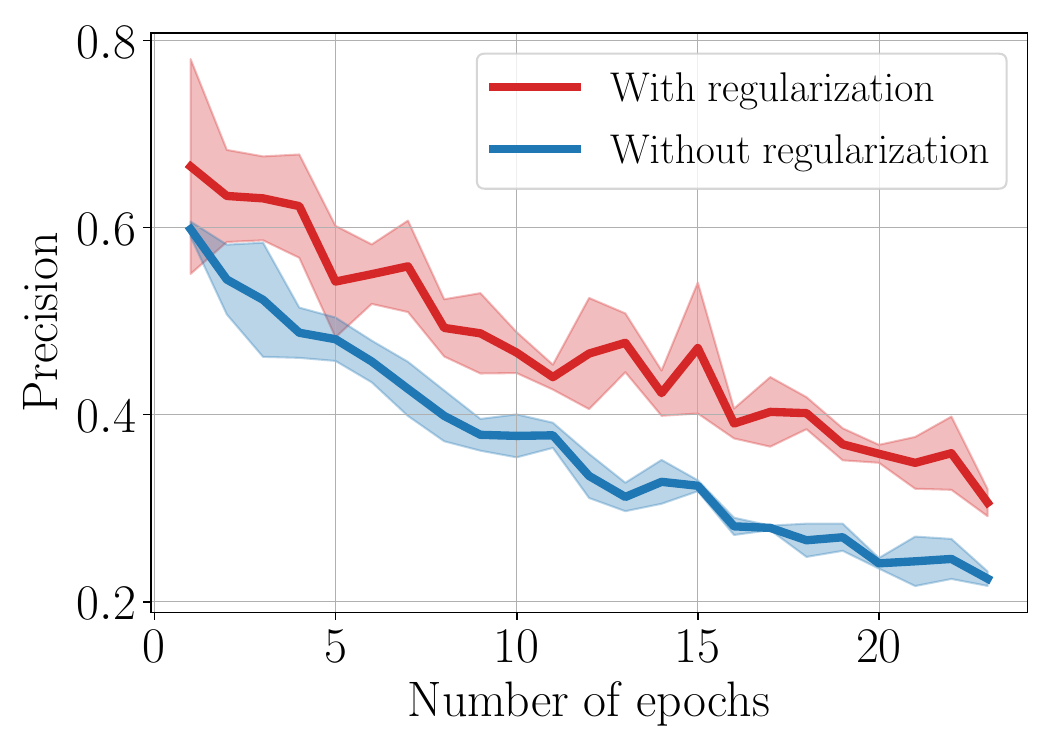}
	    \caption{Precision of self label-correction.}\label{fig:label_correction}
    \end{subfigure}
    \begin{subfigure}[b]{0.49\textwidth}
        \centering
        \includegraphics[width=0.8\textwidth]{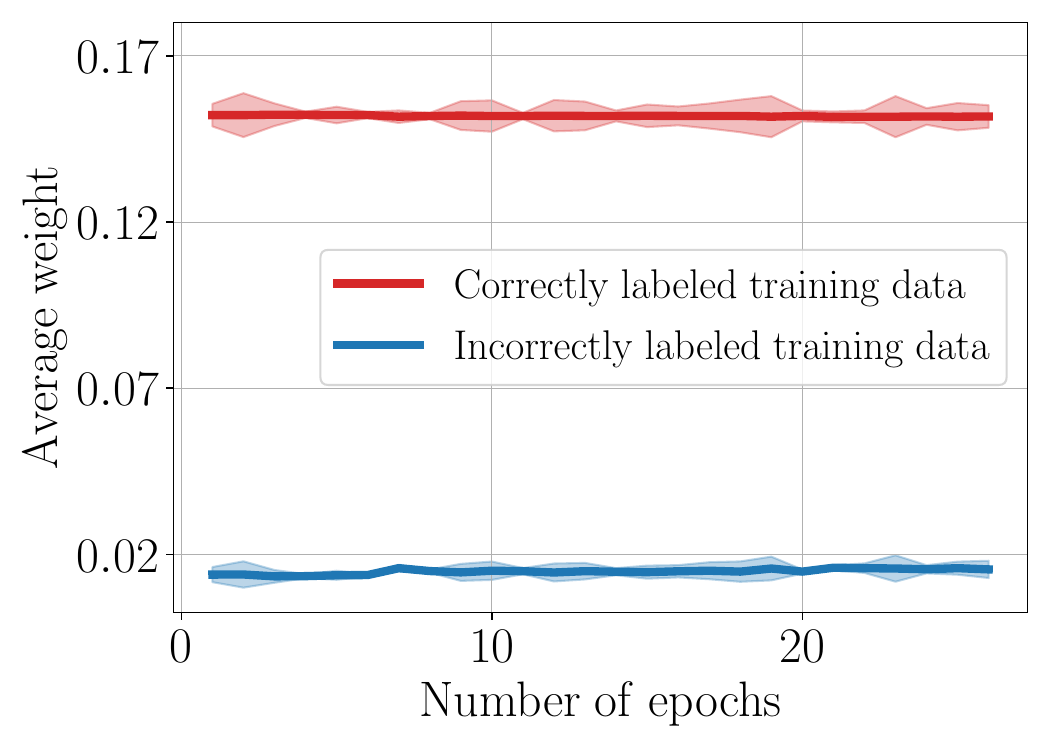}
        \caption{Average weight injected by self-label reweighting.}\label{fig:label_removal}
    \end{subfigure}
    \caption{Left: Comparing the precision of self-labeling with vs. without regularization. Precision is defined as the number of data points whose labels are corrected (cf. Line \ref{line_cor} in Algorithm \ref{alg:rnc}) divided by the number of data points whose labels are changed at every epoch. Right: Comparing the average normalized weight (cf. Line \ref{eq_rew} in Algorithm \ref{alg:rnc}) of training data points with the correct label vs. data points with an incorrect label.}
\end{figure*}

\noindent\textbf{Self-label-reweighting.} We incorporate a soft data removal step to prevent the model from overfitting to noisy labels. 
We identify data points as noisy if their loss values are large and thus down-weight them during training.
The model will prioritize data points with correct labels in its gradient by assigning smaller weights to data points with large losses.
More precisely, given a data point $(x^{(t)}_i, \tilde{y}^{(t)}_i)$, we reweight it by $\omega_i = \exp(-\ell(f_W(x^{(t)}_i), \tilde{y}^{(t)}_i)/\tau)$ where $\tau$ is a temperature parameter. 
In the implementation, we normalize the weights of every data point in every mini-batch $B$:
\begin{equation*}
	\hat\cL_B(f_W) = \frac{1}{\sum_{i\in B} \omega_i} \sum_{i\in B} \omega_i \cdot \ell(f_W(x^{(t)}_i), \tilde{y}^{(t)}_i).
\end{equation*}
Figure \ref{fig:label_removal} compares the average weight of data points with the correct label to data points with an incorrect label.
We find that, with our reweighting scheme, the average weight of incorrectly labeled data points is significantly lower than that of correctly labeled data points.
Thus, the reweighting scheme expands the gradient of correctly labeled data points, thus reducing the fraction of noisy data points in the training set.

The pseudo-code for our final approach, which combines both layer-wise regularization and self-labeling, is presented in Algorithm \ref{alg:rnc}.

\begin{algorithm}[!t]
	\caption{Regularized self-labeling ({\sc RegSL})}\label{alg:rnc}
	\begin{small}
		\textbf{Input}: Input nosiy data $(x^{(t)}_1, \tilde{y}^{(t)}_1), \dots, (x^{(t)}_{n^{(t)}}, \tilde{y}^{(t)}_{n^{(t)}})$, network model $f_W$, constraint distance $D$, distance scale factor $\gamma$,
		re-weight start step $E_r$, re-weight temperature $\tau$, label correction start step $E_c$, label correction threhold $p_t$, and fine-tuning steps $T$
		\begin{algorithmic}[1] %
			\STATE Initialize model parameters with pre-trained weights $W = \hat{W}^{s}$ and the start step $t = 0$
			\WHILE{$t < T$}
			\STATE Fetch a mini-batch of data $ \{ (x^{(t)}_i, \tilde{y}^{(t)}_i)  \}_{i=1}^{m}$
			\FOR{each data point $i$ in the batch}
			\STATE Calculate the predicted probability $\mathbf{p}_i = \text{softmax}(f_W(x^{(t)}_i))$
			\IF{$t>E_c$ \textbf{and} $ \max(\mathbf{p}_i) > p_t$ \textbf{and} $\arg \max(\mathbf{p}_i) \neq \tilde{y}^{(t)}_i $}
			\STATE Correct the label $\tilde{y}^{(t)}_i = \arg \max(\mathbf{p}_i)$ \label{line_cor}
			\ENDIF
			\STATE Calculate loss $\ell_i = \ell(f_W(x^{(t)}_i), \tilde{y}^{(t)}_i)$
			\STATE \textbf{if} $t > E_r$ \textbf{then} Calculate the weight $\omega_i = \exp(-\ell_i/\tau)$ \textbf{else} Set the weight $\omega_i = 1$ \label{eq_rew}
			\ENDFOR
			\STATE Update $f_W$ by SGD on $\cL = - \frac{1}{\sum_i \omega_i} \sum_i^m \omega_i \ell_i$ and $t = t+1$
			\STATE Project the weights $W_i$ of each layer inside the layer-wise constraint region as in Equation \ref{eq_constraint_lw}
			\ENDWHILE
		\end{algorithmic}
	\end{small}
\end{algorithm}
	\section{Experiments}\label{sec_exp}

We evaluate our proposed algorithm in a wide range of tasks and pre-trained networks. 
First, we demonstrate that our algorithm outperforms the baselines by $1.75\%$ on average across seven transfer learning tasks and can generalize to a more distant task involving medical images. 
Second, our algorithm improves the robustness of fine-tuning in the presence of noisy labels compared to previous methods. Our algorithm improves over numerous baselines by $3.56\%$ on average under both independent and correlated label noise. Moreover, our approach can achieve a similar performance boost for fine-tuning vision transformer models on noisy labeled data. 
Finally, we demonstrate that our algorithms outperform fine-tuning baselines by $0.75\%$ on average for the related task of few-shot classification. 
Our extensive evaluation confirms that our approach is applicable to a broad range of settings, thereby validating our algorithmic insights.
Due to space constraints, we defer the ablation studies to Section \ref{sec:ablation_study}. Our code is available at \href{https://github.com/VirtuosoResearch/Regularized-Self-Labeling}{https://github.com/VirtuosoResearch/Regularized-Self-Labeling}.
 
\subsection{Experimental setup}\label{sec_exp_setup}

\noindent\textbf{Data sets.}
First, we evaluate fine-tuning on seven image classification datasets and one medical image dataset. The statistics of the seven image data sets are described in Table \ref{tab: dataset_statistics}. For the medical imaging task, we consider the ChestX-ray14 dataset, which contains 112,120 frontal-view chest X-ray images labeled with 14 different diseases \cite{wang2017chestx, rajpurkar2017chexnet}.
Second, we evaluate Algorithm \ref{alg:rnc} under two different kinds of label noise to test the robustness of fine-tuning. We selected the MIT-Indoor data set~\cite{sharif2014cnn} as the benchmark data set and randomly flipped the labels of the training samples. The results for the other datasets are similar, as detailed in Section \ref{sec:extended_exp}.
We consider both independent random noise and correlated noise in our experiments.
We generate the independent random noise by flipping the labels uniformly with a given noise rate. We simulate the correlated noise by using the predictions of an auxiliary network as noisy labels. Section \ref{sec:exp_details} describes the label flipping process.
For few-shot image classification, we conduct experiments on the miniImageNet benchmark~\cite{vinyals2016matching} following the setup of~\citet{tian2020rethinking}.

\noindent\textbf{Architectures.} We use the ResNet-101~\cite{he2016deep} network for transfer learning tasks and ResNet-18~\cite{he2016deep} network for ChestX-ray data set. We use the ResNet-18 network for the label-noise experiments and extend our algorithm to the vision transformer~(ViT) model~\cite{dosovitskiy2020image}.
The ResNet models are pre-trained on ImageNet~\cite{russakovsky2015imagenet}, and the ViT model is pre-trained on the ImageNet-21k dataset~\cite{dosovitskiy2020image}.
For the few-shot learning tasks, we use ResNet-12, pre-trained on the meta-training dataset, as in previous work~\cite{tian2020rethinking}.
We set four different values of $D_i$ for the four blocks of the ResNet models in our algorithm. We describe the fine-tuning procedure and the hyperparameters in Section \ref{sec:exp_details}.

\noindent\textbf{Baselines.}
For the transfer learning tasks, we use the Frobenius norm for distance regularization. We present ablation studies to compare the Frobenius norm and other norms in Section \ref{sec:ablation_study}.
we compare our algorithm with fine-tuning with early stop~(Fine-tuning), fine-tuning with weight decay~($\ell^2$-Norm), label smoothing~(LS), $\ell^2$-SP~\cite{li2018explicit}, and $\ell^2$-PGM~\cite{gouk2020distance}.
For testing the robustness of our algorithm, in addition to the baselines described above, we adopt several baselines that have been proposed in the supervised learning setting, including GCE~\cite{zhang2018generalized}, SCE~\cite{wang2019symmetric}, DAC~\cite{thulasidasan2019combating}, APL~\cite{ma2020normalized}, ELR~\cite{liu2020early} and self-adaptive training (SAT)~\cite{huang2020self}.
We compare our results with the same baselines to evaluate transfer learning in few-shot classification tasks. We describe the implementation and hyperparameters of baselines in Section \ref{sec:exp_details}.

\subsection{Experimental results}

\noindent\textbf{Improved regularization for transfer learning.}
We report the test accuracy of fine-tuning  ResNet-101 on seven data sets in Table \ref{tab:finetuning_res}. We observe that our method performs the best among six regularized fine-tuning methods on average. Our proposed layer-wise regularization method achieves an average improvement of $\mathbf{1.76\%}$ over distance-based regularization~\cite{gouk2020distance}. In particular, our approach outperforms $\ell^2$-PGM by $\mathbf{2 \sim 3\%}$ on both Stanford-Dogs and MIT-Indoor data sets. This suggests that adding appropriate distance constraints for all the layers is better than applying only one constraint to the top layers.
In Table \ref{tab:compare_norm} (cf. Section \ref{sec:ablation_study}), we note that using the MARS norm of \citet{gouk2020distance} yields comparable results to using the $\ell_2$ norm with our approach.

Next, we apply our approach to medical image classification,  a more distant transfer task (relative to the source dataset ImageNet).
We report the mean AUROC (averaged over predicting all 14 labels) on the ChestX-ray14 data set in Table~\ref{tab:chexnet_results} (cf. Section \ref{sec:extended_exp}). With layer-wise regularization, we see an $\bm{0.39\%}$ improvement over the baseline methods.
Finally, we extend our approach to text classification. The results are consistent with the above and are presented in Table \ref{tab:text_res} (see Section \ref{sec:extended_exp}).

\begin{table*}[!t]
\centering
\caption{Top-1 test accuracy for fine-tuning  ResNet-101 pre-trained on the ILSVRC-2012 subset of ImageNet. Results are averaged over three random seeds.}\label{tab:finetuning_res}
\begin{scriptsize}
\begin{tabular}{lccccccc}
\toprule
              & Aircrafts & CUB-200-2011 & Caltech-256 & Stanford-Cars & Stanford-Dogs & Flowers & MIT-Indoor \\ \midrule
Fine-tuning   & 73.96$\pm$0.34 & 80.31$\pm$0.26 & 79.41$\pm$0.23 & 89.28$\pm$0.14          & 82.89$\pm$0.35 & 92.92$\pm$0.15 &  74.78$\pm$0.97 \\
$\ell^2$-Norm & 74.84$\pm$0.15 & 80.80$\pm$0.16 & 79.67$\pm$0.31 & 88.96$\pm$0.18          & 83.00$\pm$0.10 & 92.76$\pm$0.26 &  76.57$\pm$0.70 \\
LS            & 74.19$\pm$0.41 & 82.22$\pm$0.17 & 81.10$\pm$0.13 & \textbf{89.66$\pm$0.10} & 84.14$\pm$0.21 & 93.72$\pm$0.04 &  76.84$\pm$0.31 \\
$\ell^2$-SP   & 74.09$\pm$0.98 & 81.49$\pm$0.36 & 84.13$\pm$0.09 & 88.96$\pm$0.09          & 88.95$\pm$0.15 & 93.06$\pm$0.28 &  78.11$\pm$0.37 \\
$\ell^2$-PGM  & 74.90$\pm$0.26 & 81.23$\pm$0.32 & 83.25$\pm$0.33 & 88.92$\pm$0.39          & 86.48$\pm$0.28 & 93.23$\pm$0.34 &  77.31$\pm$0.30 \\ \midrule
{\sc RegSL} (ours)          & \textbf{75.32$\pm$0.23} & \textbf{82.24$\pm$0.21} & \textbf{84.90$\pm$0.16} & 89.14$\pm$0.22 & \textbf{89.58$\pm$0.13} & \textbf{93.82$\pm$0.35} & \textbf{79.30$\pm$0.31} \\ \bottomrule
\end{tabular}
\end{scriptsize}
\end{table*}

\begin{table*}[!t]
\centering
\caption{Top-1 test accuracy of fine-tuning ResNet-18 on the Indoor data set with various settings of noisy labels in the training set. Results are averaged over 3 random seeds.}\label{tab:noise_res}
\begin{small}
\begin{tabular}{@{}ccccccc@{}}
\toprule
\multirow{2}{*}{Data sets}    & \multirow{2}{*}{Methods} & \multicolumn{4}{c}{independent noise} & correlated noise \\
                             &                          & 20\%    & 40\%    & 60\%    & 80\%    & 25.18\%          \\ \midrule
\multirow{10}{*}{Indoor} & Fine-tuning  & 65.02$\pm$0.39 & 57.49$\pm$0.39 & 44.60$\pm$0.95 & 27.09$\pm$0.19 & 67.49$\pm$0.74 \\
                             & LS           & 67.04$\pm$0.58 & 58.98$\pm$0.57 & 48.56$\pm$0.53 & 25.82$\pm$0.90 & 68.06$\pm$0.76 \\
                             & $\ell^2$-PGM & 69.45$\pm$0.19 & 62.74$\pm$0.60 & 51.24$\pm$0.15 & 30.15$\pm$0.94 & 68.86$\pm$0.27 \\
                             & GCE          & 70.45$\pm$0.40 & 64.73$\pm$0.21 & 54.10$\pm$0.16 & 29.58$\pm$0.26 & 68.88$\pm$0.18 \\
                             & SCE          & 69.43$\pm$0.20 & 64.68$\pm$0.45 & 55.07$\pm$0.52 & 29.85$\pm$0.44 & 69.00$\pm$1.22 \\
                             & DAC          & 64.45$\pm$0.31 & 59.73$\pm$0.27 & 47.44$\pm$0.09 & 26.69$\pm$0.34 & 68.06$\pm$1.32 \\
                             & APL          & 70.05$\pm$0.41 & 66.22$\pm$0.10 & 52.51$\pm$0.66 & 30.90$\pm$0.37 & 68.31$\pm$0.77 \\
                             & SAT          & 68.98$\pm$0.63 & 63.43$\pm$0.72 & 52.84$\pm$0.38 & 29.60$\pm$0.53 & 67.06$\pm$0.55 \\
                             & ELR          & 71.43$\pm$0.80 & 66.34$\pm$0.48 & 55.22$\pm$0.73 & 31.24$\pm$0.19 & 69.38$\pm$0.49 \\ \cmidrule(l){2-7} 
                             & {\sc RegSL} (ours)         & \textbf{72.51$\pm$0.46} & \textbf{68.13$\pm$0.16} & \textbf{57.59$\pm$0.55} & \textbf{34.08$\pm$0.79} & \textbf{70.12$\pm$0.83} \\ \bottomrule
\end{tabular}
\end{small}
\vspace{-0.1in}
\end{table*}

\noindent\textbf{Improved robustness under label noise.}
Next, we report the test accuracy of our approach on the indoor dataset with independent and correlated noise in Table \ref{tab:noise_res}. 
We find that our approach consistently outperforms baseline methods by $\mathbf{1\sim3\%}$ for various settings involving label noise.
First, our method ({\sc RegSL}) improves the performance by over $\mathbf{4\%}$ on average compared to distance-based regularization~\cite{gouk2020distance}. This result implies our method is more robust to label noise in the training labels.
Second, our method outperforms previous supervised training methods by an average of $\textbf{3.56\%}$.
This result suggests that regularization is critical for fine-tuning.
Taken together, our results suggest that regularization and self-labeling complement each other during the fine-tuning process.
This is reinforced by our ablation study in Section \ref{sec:ablation_study}, where we investigate the influence of each component of our algorithm and find that removing any component degrades performance.

We further apply our approach to fine-tuning pre-trained ViT~\cite{dosovitskiy2020image} from noisy labels, under the same settings as those in Table~\ref{tab:noise_res}. 
Table~\ref{tab:vit_results} shows the result (cf. Section \ref{sec:extended_exp}).
First, we find that our approach outperforms the best regularization methods by $\bm{1.17\%}$, averaged over two settings.
Second, we find that our approach also improves upon self-labeling by $\bm{13.57\%}$ averaged over two settings.
These two results again highlight that both regularization and self-labeling contribute to the final result.
While regularization prevents the model from overfitting to the random labels, self-labeling injects the belief of the fine-tuned model into the noisy dataset.

\begin{wraptable}[13]{r}{0.45\textwidth}
\centering
\vspace{-0.17in}
\caption{Test accuracy with 95\% confidence interval for fine-tuning ResNet-12 over 600 meta-test splits of miniImageNet.}\label{tab:few_shot}
\vspace{-0.06in}
\begin{small}
\begin{tabular}{@{}lcc@{}}
    \toprule
    \multirow{2}{*}{Methods} & \multicolumn{2}{c}{miniImageNet} \\ \cmidrule(l){2-3} 
                             & 5-way-1-shot     & 5-way-5-shot   \\ \midrule
    Fine-tuning   & 60.56 $\pm$ 0.78 &  76.42 $\pm$ 0.36              \\
    $\ell^2$-Norm & 60.93 $\pm$ 0.81 &  76.57 $\pm$ 0.55              \\
    LS            & 61.31 $\pm$ 0.77 &  76.73 $\pm$ 0.62              \\
    $\ell^2$-SP   & 61.48 $\pm$ 0.76 &  77.02 $\pm$ 0.62              \\
    $\ell^2$-PGM  & 61.35 $\pm$ 0.83 &  77.33 $\pm$ 0.59              \\ \midrule
    {\sc RegSL} (ours)          & \textbf{61.71 $\pm$ 0.77} &  \textbf{78.03 $\pm$ 0.54}              \\ \bottomrule
    \end{tabular}
\end{small}
\end{wraptable}
\noindent\textbf{Extension to few-shot classification.}
We extend our approach to a few-shot image classification task.
We compare our approach to the baseline regularization methods.
In Table \ref{tab:few_shot}, we report the average accuracy of 600 sampled tasks from the meta-test split of the miniImageNet benchmark~\cite{vinyals2016matching}.
We find that our layer-wise regularization method achieves an average improvement of $\mathbf{0.75\%}$ compared to previous regularization methods.
Additionally, regularization methods generally improve the performance of fine-tuning by over $\mathbf{1\%}$.

\textbf{Ablation studies.} We study the influence of removing each component from our algorithm. Results shown in Table \ref{tab:noise_ablation} (cf. Section \ref{sec:ablation_study}) suggest that they all degrade performance. Thus, all three components in Algorithm \ref{alg:rnc} contribute to the final performance. The importance of each component depends on the noise rate. 
In particular, if the noise rate exceeds $40\%$, label correction is the most critical component. 
Other ablation studies we present include comparing the $\ell_2$ norm and the MARS norm in our algorithm, as well as comparing the distance parameters between layer-wise and constant regularization.
We leave the details to Section \ref{sec:ablation_study}.

	\section{Conclusion}

This paper studied regularization methods for fine-tuning as well as their robustness properties.
We investigated the generalization error of fine-tuning using  PAC-Bayesian techniques.
This leads to two empirical measures, which we empirically computed.
The analysis inspired us to consider layer-wise regularization for fine-tuning.
This approach performs well on both pre-trained ResNets and ViTs; we have found that fine-tuning using vanilla fine-tuning yields the best results, and we have encoded the distance patterns in layer-wise regularization.
We then evaluated the performance of regularized fine-tuning from noisy labels.
We proposed a self-label-correction and label-reweighting approach in the noisy setting.
In a follow-up paper \cite{ju2022robust}, we improve upon our theoretical result in Theorem \ref{thm_fnn} and derive an improved Hessian-based generalization bound, which leads to non-vacuous bounds when measured in practice.

\paragraph{Acknowledgment.} Thanks to the anonymous reviewers and the area chairs for carefully reading through our work and providing constructive feedback.
Thanks to the program committee chairs for communicating with us regarding an omitted factor in equation \eqref{eq_concentra}, which we have added in the current version.
H. Z. would like to acknowledge Sen Wu and Christopher R\'e for several discussions related to this work.
This work has been partially supported by Northeastern University's Khoury seed/proof-of-concept grant.

	\bibliographystyle{plainnat}
	\bibliography{rf,rf_tl,rf_ldy}
    \newpage
	\appendix
	\section{Proof of Theorem \ref{thm_fnn}}\label{sec:proof_details}

\textbf{Background.}
PAC-Bayesian generalization theory provides a useful approach to incorporating data-dependent properties, such as noise stability and sharpness, into generalization bounds.
Recent works (e.g., \citet{bartlett2017spectrally,neyshabur2017pac}) have introduced generalization bounds for multi-layer neural networks, aiming to explain why neural network models generalize well despite having more trainable parameters than the number of training examples. On the one hand, the VC dimension of a neural network is known to be roughly equal to its number of parameters \cite{bartlett2019nearly}. On the other hand, the number of parameters is not a good capacity measure for neural nets, as evidenced by the popular work of \citet{zhang2016understanding}. The bounds of \citet{bartlett2017spectrally} (and subsequent works) provide a more meaningful notion of “capacity” compared to VC dimension.

While these bounds constitute an improvement over classical learning theory, it remains unclear whether they are tight or non-vacuous. One response is a computational framework from \citet{dziugaite2017computing}. That work demonstrates that directly optimizing the PAC-Bayes bound yields a significantly smaller bound and lower test error simultaneously (see also \citet{zhou2018non} for a large-scale study). The recent work of \citet{jiang2019fantastic} further compared different complexity notions and noted that the ones given by PAC-Bayes tools correlate better with empirical performance.
In particular, we will use the following classical result \cite{mcallester1999pac,mcallester1999some}.

\begin{theorem}[Theorem 1 in \citet{mcallester1999pac}]\label{thm_pac_bayes}
	Let $\cH$ be some hypothesis class. Let $P$ be a prior distribution on $\cH$ that is independent of the training set.
	Let $Q_S$ be a posterior distribution on $\cH$ that may depend on the training set $S$.
	Suppose the loss function is bounded from above by $C$.
	With probability $1 - \delta$ over the randomness of the training set, the following holds
	\begin{align}\label{thm_pac}
		\exarg{h \sim Q_S}{\cL(h)} \le {\exarg{h \sim Q_S}{\hat \cL(h)} + C\sqrt{\frac{{\KL(Q_S \mid\mid P) + 3\ln \frac{n}{\delta} + 8}}{n}}}.
	\end{align}
\end{theorem}

We remark that the original statement in \citet{mcallester1999pac} requires that the loss function is bounded between $0$ and $1$.
The above statement modifies the original statement and instead applies to a loss function bounded between $0$ and $C$, for some fixed constant $C > 0$.
This is achieved by rescaling the loss by $1/C$, leading to the $C$ factor in the right-hand side of equation \eqref{thm_pac}.
To invoke the above result in our setting, we set the prior distribution $P = \cN(\hat{W}^{(s)}, \sigma^2 \id)$, where $\hat{W}^{(s)}$ are the weights of the pre-trained network.
The posterior distribution $Q_S$ is centered at the fine-tuned model as $\cN(\hat{W}, \sigma^2 \id)$.
Based on the above result, we present the proof of Theorem \ref{thm_fnn}.

\begin{proof}[Proof of Theorem \ref{thm_fnn}]
	First, we show that the KL divergence between $P$ and $Q_S$ is equal to $\frac{1}{2\sigma^2}\norm{\hat W - \hat W^{(s)}}^2$.
	We expand the definition using the density of multivariate normal distributions.
	{\begin{align*}
		\KL(P \mid\mid Q_S) &= \exarg{W \sim \cP}{\log\bigbrace{\frac{\Pr(W \sim P)}{\Pr(W \sim Q_S)}}} \\
		&= \exarg{W \sim \cP}{\log{\frac{\exp(-\frac{1}{2\sigma^2} \norm{W - \hat W^{(s)}}^2)}{\exp(-\frac{1}{2\sigma^2} \norm{W - \hat{W}}^2)}}} \\
		&= -\frac{1}{2\sigma^2}\exarg{W\sim\cP}{\norm{W - \hat{W}^{(s)}}^2 - \norm{W - \hat{W}}^2} \\
		&= \frac{1}{2\sigma^2}\exarg{W\sim\cP}{\inner{\hat{W}^{(s)} -\hat{W}}{2W - \hat{W}^{(s)} - \hat{W}}} \\
		&= \frac{1}{2\sigma^2}\normFro{\hat{W} - \hat{W}^{(s)}}^2 \le \frac{\sum_{i=1}^L D_i^2}{2\sigma^2},
	\end{align*}}%
	where the last line uses the fact that $\normFro{\hat{W}_i - \hat{W}_i^{(s)}} \le D_i$, for all $1\le i\le L$.

	Next, let the dimension of $W_i$ be $d_{i-1}$ times $d_{i}$, for every $i = 1, 2, \dots, L$.
	In particular, $d_0$ is equal to the dimension of the input data points. 
	Let $H = \max_{0 \le i \le L} d_i$ be the maximum matrix dimension size across all the $L$ layers.
	Let $e = \sigma\sqrt{2H\log(2L\cdot H/\delta)}$.
	We show that for any $\delta > 0$ and any $L$-layer feedforward neural network parameterized by $W = [W_1, W_2, \dots, W_L]$,
	with probability at least $1 - \delta$, the perturbation $Q_S$ increases the loss function $\hat{\cL}(h)$ by at most
	\begin{align}\label{eq_perturb}
	    C_1\cdot e \cdot \bigbrace{\sum_{i=1}^L \frac{{\prod_{j=1}^L\big(\norm{W_j}_2 + e\big)}}{\norm{W_i}_2 + e}}.
	\end{align}
	Consider some data point $(x, y)$ evaluated at a neural network function $f_W$.
	Let $U_i$ be a Gaussian perturbation matrix on $W_i$ with mean zero and entrywise variance $\sigma^2$, for $i = 1,\dots,L$.
	Denote by $U = [U_1, U_2, \dots, U_L]$.
	Since $\ell(x, y)$ is 1-Lipschitz over $x$, we get
	\begin{align}\label{eq_lipschitz}
		\abs{\ell(f_{W + U}(x), y) - \ell(f_W(x), y)} \le \norm{f_{W + U}(x) - f_W(x)}_2.
	\end{align}
	{Since $U_i \in \real^{d_{i-1} \times d_i}$ is a random matrix with i.i.d. entries sampled from the standard Gaussian distribution, we can upper bound the operator norm of $U_i$ by applying well-known concentration results.
	In particular, by equation (4.1.8) and section (4.2.2) of \citet{tropp2015introduction}, we get 
	\begin{align}
	    \Pr\Big[\frac 1 {\sigma} \norm{U_i}_2 \ge  t\Big] \le (d_{i-1} + d_i) \cdot \exp\Big(-\frac{-t^2}{2\max(d_{i-1}, d_i)}\Big). \label{eq_tail_prob}
	\end{align}

	Thus, for all $i = 1, 2, \dots, L$, with probability at most $\delta / L$, we have that
	\begin{align}
	    \frac 1 {\sigma} \norm{U_i}_2 \le \sqrt{2H\log(2L \cdot H / \delta)}. \label{eq_tail_operator}
	\end{align}}%
	In particular, the right hand side above is obtained by setting $2H \exp\big(-\frac{t^2}{2H} \big) = \delta / L$ and solve for $t = \sqrt{2H \log(2L\cdot H / \delta)}$.
	Since the right-hand side of equation \eqref{eq_tail_prob} is at most $2H\exp\big(-\frac{t^2}{2H}\big)$, by the union bound, we conclude that with probability at most $1 - \delta$, for all $i = 1, 2,\dots, L$,
	the operator norm of $U_i$ is at most $\sigma \cdot \sqrt{2H \log(2L\cdot H / \delta)}$.
	
	Let $z_i$ be the input to the $i$-th layer of $f_W$.
	Let $\tilde z_i$ be the input to the $i$-th layer of $f_{W + U}$.
    We can expand the right-hand side of equation \eqref{eq_lipschitz} as
    \begin{align}
        & \bignorm{\psi_L((W_L + U_L) \tilde z_L) - \psi_L(W_L z_L)}_2 \label{eq_psi_diff_L} \\
        \le& \bignorm{(W_L + U_L) \tilde z_L - W_L z_L}_2 \tag{since $\psi_L(\cdot)$ is 1-Lipschitz}  \\
        \le& \bignorm{W_L(\tilde z_L - z_L)}_2 + \bignorm{U_L \tilde z_L}_2 \tag{by triangle inequality} \\
        \le& \norm{W_L}_2 \cdot \bignorm{\tilde z_L - z_L}_2 + {\bignorm{U_L \tilde z_L}_2} \tag{$\norm{Wx}_2 \le \norm{W}_2 \cdot \norm{x}_2$ for any vector $x$} \\
        \le& \norm{W_L}_2  \cdot \norm{\tilde z_L - z_L}_2 + \sigma\cdot \sqrt{2H\log(2L \cdot H / \delta)} \cdot \norm{\tilde z_L}_2, \label{eq_concentra}
    \end{align}
    where the last step is by equation \eqref{eq_tail_operator}.
    Recall from Section \ref{sec_setup} that the $i$-th layer takes an input $z$ and outputs $\psi_i(W_i z)$, for every $i = 1, \dots, L$.
    Because we have assumed that $\psi_i(\cdot)$ is 1-Lipschitz, for all $i = 1, \dots, L-1$, we can bound
    \begin{align*}
        \norm{\tilde{z}_L}_2
        = \phi_{L-1} \circ \phi_{L - 2} \cdots \circ \phi_1(x)
        &\le \Big(\prod_{i=1}^{L-1} \norm{W_i + U_i}_2\Big) \cdot \norm{x}_2 \\
        &\le \Big(\prod_{i=1}^{L-1} \norm{W_i + U_i}_2\Big) \cdot C_1 \tag{since $\norm{x}_2 \le C_1$ for  $x\in\cX$} \\
        &\le \Big(\prod_{i=1}^{L-1} (\norm{W_i}_2 + e) \Big)\cdot C_1. \tag{by equation \eqref{eq_tail_operator}}
    \end{align*}
    Applying the above to equation \eqref{eq_concentra}, we have shown
    \begin{align}
        &\norm{f_{W+ U}(x) - f_W(x)}_2 \nonumber \\
        \le& \norm{W_L}_2 \cdot \norm{\tilde{z}_L - z_L}_2 + \sigma \cdot \sqrt{2H \log(2L \cdot H / \delta)} \cdot \Big(\prod_{i=1}^{L-1} (\norm{W_i}_2 + e)\Big) \cdot C_1. \label{eq_rec_begin}
    \end{align}
    Next, we can expand the difference between $\tilde{z}_L = \psi_{L-1}((W_{L-1} + U_{L-1}) \tilde{z}_{L-1})$ and $z_L = \psi_{L-1}(W_{L-1} {z}_{L-1})$ similar to equation \eqref{eq_psi_diff_L}.
    In general, for any $i = 1, 2, \dots, L$, we get
    \begin{align}
        &\norm{\tilde{z}_{i} - z_i}_2 \nonumber \\
        \le& \norm{W_{i-1}}_2 \cdot \norm{\tilde{z}_{i-1} - z_{i-1}}_2 + \sigma \cdot \sqrt{2H \log(2L \cdot H / \delta)} \cdot \Big(\prod_{i=1}^{i-2} (\norm{W_i}_2 + e) \Big) \cdot C_1. \label{eq_diff_rec}
    \end{align}
    By repeatedly applying equation \eqref{eq_diff_rec} together with equation \eqref{eq_rec_begin} (relaxing $\norm{W_i}_2$ to $\norm{W_i}_2 + e$), we can show that with probability at least $1 - \delta$,
    \begin{align*}
        \norm{f_{W + U}(x) - f_W(x)}_2
        \le \sigma \cdot C_1 \sqrt{2H \log(2L \cdot H / \delta)} \cdot
        \Big(\sum_{i=1}^L \frac{\prod_{j=1}^L (\norm{W_j}_2+ e)}{ \norm{W_i}_2 + e}\Big).
    \end{align*}
    Combined with equation \eqref{eq_lipschitz}, we have shown that equation \eqref{eq_perturb} holds.
    Finally, with equation \eqref{eq_perturb} and the fact that that the loss function $\ell(\cdot)$ is bounded from above by some fixed value $C_2$, we can bound the expectation of $\hat{\cL}(h)$ for some $h \sim Q_S$ as
    \begin{align}
        \exarg{h\sim Q_S}{\hat{\cL}(h)}
        \le \hat{\cL}(f_{\hat W}) &+ (1 - \delta) \cdot \sigma \cdot C_1 \sqrt{2H \log\Big(\frac {2L \cdot H} {\delta}\Big)} \cdot \Big(\sum_{i=1}^L \frac{\prod_{j=1}^L (\norm{\hat W_j}_2 + e)}{\norm{\hat W_i}_2 + e} \Big) \nonumber \\
        &+ \delta \cdot C_2. \label{eq_delta_combo}
    \end{align}
    Note that the above inequality applies for any (small) value of $\delta$.
    Thus, we can minimize the right-hand side by appropriately setting $\delta$.
    For our purpose, we will show that there exists some $\delta$ such that the right-hand side of equation \eqref{eq_delta_combo} (leaving out $\hat{\cL}\big(f_{\hat W}\big))$ is at most
    \begin{align}\label{eq_conclude_perturb}
        2 \sigma C_1 \Big(\sum_{i=1}^L \frac{\prod_{j=1}^L (\norm{\hat W_j}_2 + e)}{\norm{\hat W_i}_2 + e}\Big) \cdot \sqrt{2H \log\big({4L\cdot H \cdot C_2}\big)}.
    \end{align}
    To see that the above is true, we will abstract away the precise scalars and instead write the right-hand side of equation \eqref{eq_delta_combo} as
    $g(\delta) = (1 - \delta) A_1 \sqrt{\log (A_2 / \delta)} + C_2 \delta$, for some fixed $A_1$ and $A_2$ that does not depend on $\delta$.
    We note that subject to
    \begin{align}\label{eq_g_constraint}
        C_2 \delta \le (1 - \delta) A_1 \sqrt{\log(A_2 /\delta)},
    \end{align}
    $g(\delta)$ is at most
    \begin{align} g(\delta) \le 2(1 - \delta) A_1 \sqrt{\log(A_2/\delta)} \le 2 A_1 \sqrt{\log\Big(\frac{A_2}{\delta}\Big)}.\label{eq_g_delta}
    \end{align}
    Therefore, we only need to lower bound $\delta$ based on the constraint \eqref{eq_g_constraint}.
    In particular, the largest possible $\delta^{\star}$ under constraint \eqref{eq_g_constraint} is achieved when both sides equal:
    \begin{align*}
        C_2 \delta^{\star} = (1 - \delta^{\star}) A_1 \sqrt{\log(A_2 /\delta^{\star})}
        \ge \frac{A_1}{2} \sqrt{\log(A_2)},
    \end{align*}
    which implies that $\delta^{\star} \ge \frac{A_1}{2C_2}\sqrt{\log(A_2)}$.
    Thus, we have shown that $g(\delta^{\star})$ is less than equation \eqref{eq_conclude_perturb}, using equation \eqref{eq_g_delta} and the lower bound on $\delta^{\star}$. Additionally, $A_1$ is on the order of a fixed constant based on $\sigma$ defined below.
    
    Using a similar argument since our bound on the perturbed loss holds point-wise for every $x\in\cX$, we can likewise prove that with probability $1 - \delta$,
    \begin{align*}
        \exarg{h\sim Q_S}{{\cL}(h)} \ge \cL(f_{\hat W}) - 2\sigma C_1 \Big(\sum_{i=1}^L \frac{\prod_{j=1}^L (\norm{\hat{W}_j}_2 + e)}{\norm{\hat{W}_i}_2 + e} \Big) \cdot \sqrt{2H \log(4L \cdot H \cdot C_2)}.
    \end{align*}
    By applying equation \eqref{eq_conclude_perturb} and the above to equation \eqref{thm_pac} in Theorem \ref{thm_pac_bayes}, we thus conclude that for any $\sigma$, with probability at least $1 - 2\delta$ over the training data set, the following holds
    \begin{align}
        \cL(f_{\hat W})
        \le \hat\cL(f_{\hat W})
        &+ 4\sigma \cdot C_1 \cdot \sqrt{2H \log(4L \cdot H \cdot C_2)} \bigbrace{\sum_{i=1}^L \frac{\prod_{j=1}^L (\norm{\hat W_j}_2 + e)}{\norm{\hat W_i}_2 + e}} \nonumber \\
        &+ C_2\sqrt{\frac{\sum_{i=1}^L \frac{D_i^2} {2\sigma^2} + 3\ln{\frac {n^{(t)}} {\delta}} + 8}{n^{(t)}}}. \label{eq_close_final}
    \end{align}
    Since $\norm{\hat W_i^{(s)}}_2 \le B_i$ and $\normFro{\hat{W}_i - \hat{W}_i^{(s)}} \le D_i$, for any $i = 1,\dots,L$, we have that $\norm{\hat{W}_i}_2 \le B_i + D_i$.
    Thus, we can upper bound the above by replacing every $\norm{\hat W_i}_2$ with $B_i + D_i$.
    
    By setting 
    \[ \sigma = \frac{\varepsilon} {6\sigma C_1   \cdot\alpha\sqrt{2H\log(4LHC_2)}}, \text{ where } \alpha = \Big(\sum_{i=1}^L\frac{\prod_{j=1}^L (B_j + D_j)}{B_i + D_i}\Big), \] 
    we get $e\le \frac 1 {6\alpha}$ since $C_1 \ge 1$.
    To finish the proof, we will show that
    {$$
    \sum_{i=1}^L \frac{\prod_{j=1}^L (B_j + D_j + e)}{B_i + D_i + e} \le 3\alpha/2.$$}%
    In particular, we will show that for every $k =1,2,\dots,L$, $e \le \frac{B_k + D_k} {6L}$.
    To see this, recall that $B_k \ge 1$ and $D_k \ge 0$.
    Therefore, $\alpha \ge L$ for every $k$ and $e \le \frac{1}{6\alpha}\le \frac{B_k + D_k} {6L}$.
    This implies
    {$$\sum_{i=1}^L \frac{\prod_{j=1}^L (B_j + D_j + e)}{(B_i + D_i + e)} \le \Big(1 + \frac 1 {6L}\Big)^{L-1} \sum_{i=1}^L \frac{\prod_{j=1}^L (B_j + D_j)}{B_i + D_i} \le \frac3 2 \alpha.$$}%
    Thus, plugging in the value of $\sigma$ above, we conclude that the second term in equation \eqref{eq_close_final} is at most $\varepsilon$.
    By applying the value of $\sigma$ to the third term in equation \eqref{eq_close_final}, we conclude
    \begin{align*}
        \cL(f_{\hat{W}}) \le& \hat{\cL}(f_{\hat{W}})
        + \varepsilon \\
        &+ C_2\sqrt{\frac{
        \frac{36}{\varepsilon^2} \cdot C_1^2 H \log(4L H C_2)\Big(\sum_{i=1}^L \frac{\prod_{j=1}^L (B_j + D_j)}{B_i + D_i}\Big)^2 \Big(\sum_{i=1}^L D_i^2\Big)
        + 3 \ln \frac{n^{(t)}}{\delta} + 8}{n^{(t)}}}.
    \end{align*}
    Thus, we have proved equation \eqref{eq_pac_main} is true. The proof is complete.
\end{proof}

\paragraph{Implication}. Our result implies a tradeoff in setting the distance parameter $D_i$. While a larger $D_i$ increases the network’s ``capacity'', the generalization bound gets worse as a result.

\paragraph{Remark.} In a follow-up paper \cite{ju2022robust}, we further improve on the result of Theorem \ref{thm_fnn} and develop a Hessian-based generalization bound for fine-tuning.
	\section{Extended Experimental Setup and Results}

\subsection{Additional details of the experimental setup}\label{sec:exp_details}

We provide further details about our experiment setup. First, we introduce the data sets in our experiments and describe the pre-processing steps. Second, we present the model architectures and the fine-tuning procedures. Third, we provide a detailed description of the baselines. Finally, we describe the implementation and the hyperparameters for our algorithm and the baselines.

\begin{table}[!t]
\centering
\caption{Top-1 test accuracy on CUB-200-2011 and Flowers dataset with independent label noise injected in the training set. Results are averaged over 3 random seeds.}\label{tab:other_noise_res}
\begin{tabular}{@{}cccccc@{}}
\toprule
\multirow{2}{*}{Datasets}    & \multirow{2}{*}{Methods} & \multicolumn{4}{c}{independent noise} \\
                             &                          & 20\%    & 40\%    & 60\%    & 80\%          \\ \midrule
\multirow{10}{*}{CUB-200-2011} & Fine-tuning  & 68.28 $\pm$ 0.34 & 56.88 $\pm$ 0.38 & 39.54 $\pm$ 0.42 & 16.21 $\pm$ 0.38 \\
                               & LS           & 69.89 $\pm$ 0.85 & 59.18 $\pm$ 0.43 & 41.58 $\pm$ 0.45 & 16.96 $\pm$ 0.45 \\
                               & $\ell^2$-PGM & 69.71 $\pm$ 0.31 & 58.91 $\pm$ 0.42 & 41.52 $\pm$ 0.74 & 16.76 $\pm$ 0.41 \\
                               & GCE          & 69.54 $\pm$ 0.25 & 60.15 $\pm$ 0.93 & 41.84 $\pm$ 0.47 & 17.77 $\pm$ 0.32 \\
                               & SCE          & 70.68 $\pm$ 0.67 & 61.33 $\pm$ 0.57 & 43.67 $\pm$ 0.37 & 17.62 $\pm$ 0.19 \\
                               & DAC          & 68.58 $\pm$ 0.25 & 57.37 $\pm$ 0.29 & 40.00 $\pm$ 0.37 & 15.92 $\pm$ 0.83 \\
                               & APL          & 70.59 $\pm$ 0.18 & 59.68 $\pm$ 0.69 & 40.46 $\pm$ 0.24 & 14.69 $\pm$ 0.36 \\
                               & SAT          & 68.69 $\pm$ 0.38 & 57.34 $\pm$ 0.18 & 38.75 $\pm$ 0.40 & 15.14 $\pm$ 0.20 \\
                               & ELR          & 69.92 $\pm$ 0.14 & 58.69 $\pm$ 0.02 & 40.76 $\pm$ 0.68 & 16.68 $\pm$ 0.79 \\ \cmidrule(l){2-6} 
                               & {\sc RegSL} (ours)         & \textbf{71.76 $\pm$ 0.49} & \textbf{62.79 $\pm$ 0.23} & \textbf{45.85 $\pm$ 0.66} & \textbf{17.88 $\pm$ 0.25} \\ \midrule
\multirow{10}{*}{Flowers}     & Fine-tuning   & 83.13 $\pm$ 0.15 & 72.23 $\pm$ 0.40 & 55.27 $\pm$ 0.32 & 29.35 $\pm$ 0.74 \\
                              & LS            & 83.62 $\pm$ 0.30 & 72.35 $\pm$ 0.47 & 54.23 $\pm$ 0.24 & 28.60 $\pm$ 0.45 \\
                              & $\ell^2$-PGM  & 83.45 $\pm$ 0.70 & 73.24 $\pm$ 0.20 & 56.51 $\pm$ 0.40 & 31.04 $\pm$ 0.88 \\
                              & GCE           & 83.09 $\pm$ 0.22 & 71.74 $\pm$ 0.65 & 54.73 $\pm$ 0.90 & 28.38 $\pm$ 0.95 \\
                              & SCE           & 83.45 $\pm$ 0.14 & 73.11 $\pm$ 0.05 & 55.96 $\pm$ 0.97 & 29.22 $\pm$ 0.07 \\
                              & DAC           & 83.40 $\pm$ 0.27 & 72.73 $\pm$ 0.34 & 55.27 $\pm$ 0.40 & 28.95 $\pm$ 0.89 \\
                              & APL           & 83.42 $\pm$ 0.11 & 72.06 $\pm$ 0.23 & 54.96 $\pm$ 0.36 & 28.86 $\pm$ 0.64 \\
                              & SAT           & 82.49 $\pm$ 0.25 & 72.52 $\pm$ 0.28 & 55.10 $\pm$ 0.50 & 27.82 $\pm$ 0.38 \\
                              & ELR           & 83.38 $\pm$ 0.38 & 72.53 $\pm$ 0.48 & 55.74 $\pm$ 0.42 & 30.17 $\pm$ 0.26 \\ \cmidrule(l){2-6} 
                              & {\sc RegSL} (ours)          & \textbf{83.79 $\pm$ 0.39} & \textbf{73.32 $\pm$ 0.66} & \textbf{57.52 $\pm$ 0.15} & \textbf{31.09 $\pm$ 0.25} \\ \bottomrule
\end{tabular}
\end{table}

\textbf{Datasets.} 
We evaluate fine-tuning on seven image classification data sets covering multiple applications, including fine-grained object recognition, scene recognition, and general object recognition.
We divide each dataset into a training set, a validation set, and a test set for data splitting.
We adopt the standard splitting of data given in the Aircraft dataset. For the Caltech-256 dataset~\cite{griffin2007caltech}, we use the setting described in~\citet{li2018explicit}, which randomly samples 30, 20, and 20 images of each class for the training, validation, and test sets, respectively. For other datasets, we split $10\%$ of the training set as the validation set and use the standard test set. We describe the statistics of the seven data sets in Table \ref{tab: dataset_statistics}.
When fine-tuning on the seven data sets, we preserve the original pixel and resize the shorter side to 256 pixels. The image samples are normalized with the mean and standard deviation values over ImageNet data~\cite{russakovsky2015imagenet}. Moreover, we apply commonly used data augmentation methods on the training samples, including random scale cropping and random flipping. The images are resized to $224 \times 224$ for use as the model's input. The training and batch sizes are both $16$.

To test the robustness of fine-tuning, we evaluate Algorithm \ref{alg:rnc} under two different settings with label noise. 
We select MIT-Indoor dataset~\cite{sharif2014cnn} as the benchmark dataset and randomly flipped the labels of the training samples. The results for the other datasets are similar, as shown in Appendix \ref{sec:extended_exp}.
We consider two scenarios of label noise in our experiments, independent random noise and correlated noise. Random label noise is generated by uniformly flipping the labels of a given proportion of training samples to other classes. For the correlated noise setting, the label noise is dependent on the sample. We simulate the correlated noisy label by training an auxiliary network on a held-out dataset to a certain accuracy. We then use the prediction of the auxiliary network as noisy labels. 
For few-shot learning, we conduct experiments on the miniImageNet~\cite{vinyals2016matching} few-shot image recognition benchmark and follow the meta-training and meta-test setup described in \citet{tian2020rethinking}.

\textbf{Architectures.} For image data sets, we use ResNet-101 network which is pre-trained on ImageNet dataset~\cite{russakovsky2015imagenet}.
For the label-noise experiments, we use the ResNet-18 network. In the transfer learning and label-noise experiments, we fine-tune the model using the Adam optimizer with an initial learning rate of $0.0001$ for $30$ epochs and decay the learning rate by $0.1$ every $10$ epochs. We report the average Top-1 accuracy on the test set of $3$ random seeds.
For the few-shot learning tasks, we use ResNet-12, pre-trained on the meta-training dataset, as in previous work~\cite{tian2020rethinking}.
To fine-tune on the meta-test sets, we use Adam optimizer with an initial learning rate $5e^{-5}$ and fine-tune the model on the training set for $25$ epochs. We report the average classification accuracies on 600 sampled tasks from the meta-test split. 

\textbf{Baselines.}
For the transfer learning tasks, we focus on using the Frobenius norm for distance regularization. We present ablation studies to compare the Frobenius norm and other norms in Appendix \ref{sec:extended_exp}.
we compare our algorithm with fine-tuning with early stop~(Fine-tuning), fine-tuning with weight decay~($\ell^2$-Norm), label smoothing~(LS) $\ell^2$-SP~\cite{li2018explicit}, and $\ell^2$-PGM~\cite{gouk2020distance}.
For testing the robustness of our algorithm, in addition to the baselines described above, we adopt several baselines that have been proposed in the supervised learning setting, including GCE~\cite{zhang2018generalized}, SCE~\cite{wang2019symmetric}, DAC~\cite{thulasidasan2019combating}, APL~\cite{ma2020normalized}, ELR~\cite{liu2020early} and self-adaptive training (SAT)~\cite{huang2020self}.
We compare our results with the same fine-tuning baselines used in transfer learning for few-shot classification tasks.

\begin{table}[!t]
\centering
\caption{Left: Mean AUROC of fine-tuning ResNet-18 on the ChestX-ray14 dataset. Right: Top-1 accuracy of fine-tuning ViT on the indoor dataset. Results are averaged over 3 random seeds.}
\begin{minipage}{0.48\textwidth}
\centering
\subcaption{The ChestX-ray14 dataset.}\label{tab:chexnet_results}
\begin{small}
\begin{tabular}{@{}lc@{}}
    \toprule
    Methods & Mean AUROC  \\ \midrule
    Fine-tuning   & 0.8159 $\pm$ 0.0667             \\
    $\ell^2$-Norm & 0.8198 $\pm$ 0.0644 \\
    LS            & 0.7885 $\pm$ 0.0578 \\
    $\ell^2$-SP   & 0.8231 $\pm$ 0.0658\\
    
    $\ell^2$-PGM   & 0.8235 $\pm$ 0.0636           \\ \midrule
    {\sc RegSL} (ours)          & \textbf{0.8274 $\pm$ 0.0654}      \\ \bottomrule
    \end{tabular}
\end{small}
\end{minipage}\hfill
\begin{minipage}{0.48\textwidth}
\centering
\subcaption{The indoor dataset.}\label{tab:vit_results}
\begin{small}
\begin{tabular}{@{}lcc@{}}
    \toprule
    \multirow{2}{*}{Methods} & \multicolumn{2}{c}{Independent Noise} \\ \cmidrule(l){2-3} 
                             & 40\%     & 80\%   \\ \midrule
    Fine-tuning   & 75.87 $\pm$ 1.15 &  32.06 $\pm$ 3.17              \\
    Regularization   & 82.66 $\pm$ 0.44  &  63.01 $\pm$ 0.83             \\
    Self-labeling  & 79.13 $\pm$ 0.37 &  41.69 $\pm$ 0.16              \\ \midrule
    {\sc RegSL} (ours)          & \textbf{83.48 $\pm$ 0.29}  &  \textbf{64.50 $\pm$ 0.53}              \\ \bottomrule
    \end{tabular}
\end{small}
\end{minipage}
\end{table}

\begin{table*}[!t]
\centering
\caption{Top-1 accuracy of fine-tuning a three-layer feed forward network pre-trained on SST with no label noise and with independent label noise. Results are averaged over 5 random seeds.}\label{tab:text_res}
\begin{small}
\begin{tabular}{lccccc}
\toprule
Noise Rate 0\%  & MR & CR & MPQA & SUBJ & TREC  \\ \midrule
Fine-tuning   & 83.37$\pm$0.70 & 83.29$\pm$0.80 & 87.56$\pm$0.70 & 93.14$\pm$0.42 & 83.28$\pm$0.86 \\
$\ell^2$-PGM  & 84.16$\pm$0.41 & 83.87$\pm$0.66 & 87.77$\pm$0.62 & 93.16$\pm$0.21 & \textbf{84.48$\pm$0.52} \\ 
{\sc RegSL} (ours)          & \textbf{84.20$\pm$0.47} & \textbf{84.35$\pm$0.60} & \textbf{87.95$\pm$0.65} & \textbf{93.50$\pm$0.17} & 83.73$\pm$0.75  \\ \midrule
Noise Rate 40\%  & MR & CR & MPQA & SUBJ & TREC  \\ \midrule
Fine-tuning   & 82.91$\pm$0.25 & 77.67$\pm$1.11 & 82.85$\pm$0.98 & 90.78$\pm$0.88 & 70.80$\pm$0.22 \\
$\ell^2$-PGM  & 83.45$\pm$0.28 & 79.47$\pm$0.37 & 84.03$\pm$0.41 & 72.14$\pm$0.21 & 73.48$\pm$0.90 \\
{\sc RegSL} (ours)          & \textbf{83.54$\pm$0.40} & \textbf{79.89$\pm$0.51} & \textbf{84.15$\pm$0.99} & \textbf{72.48$\pm$0.45} & \textbf{74.80$\pm$0.87}  \\ \bottomrule
\end{tabular}
\end{small}
\end{table*}

\textbf{Implementation and hyperparameters.} 
For all baselines and datasets, all regularization hyperparameters are optimized using the Optuna optimization framework on the validation dataset.
In our proposed algorithm, we search the constraint distance $D$ in $[0.05, 10]$ and the distance scale factor $\gamma$ in $[1, 5]$ by sampling.
For setting layer-wise constraints, we set four different constraints for the four blocks of the ResNet. Specifically, ResNet-101 consists of four blocks, each with 10, 13, 67, and 10 convolutional layers, respectively. For every block, we set the same constraint value for all the layers within the block. Therefore, there are four distance values $(D, D\cdot\gamma, D\cdot\gamma^2, D\cdot\gamma^3)$ for each block. Similar constraints are applied to the ResNet-18 and ResNet-12 models in the experiments.
The search space for other hyperparameters is shown as follows. 
The reweighting step begins at $E_r$, which is searched in the set $\{3, 5, 8, 10, 13\}$.
The reweighting temperature factor $\gamma$ is searched in $\{3.0, 2.0, 1.5, 1.0\}$.
Label correction start step $E_c$ is searched in $\{5, 8, 10, 13, 15\}$. 
The label correction threshold $p$ is set as $0.90$ in all experiments. The validation set sizes used for evaluating these hyperparameter choices range from 536 to 5120 (cf. Table \ref{tab: dataset_statistics}). 
For the results in Table \ref{tab:finetuning_res}, we search for 20 different trials of the hyperparameters for the proposed algorithm and the baselines. 
For the results in Table \ref{tab:noise_res}, we search 20 times on the self-labeling parameters and 20 times on the regularization parameters. For the few-shot learning experiments in Table \ref{tab:few_shot}, we use the validation set of Mini-ImageNet, following the training procedure of \citet{dhillon2019baseline} and search for 20 different trials.

We report the results for baselines by running the official open-sourced implementations. We describe the hyperparameter space for baselines as follows. 
For GCE~\cite{zhang2018generalized},  we search the factor $q$ in their proposed truncated loss $\cL_q$ in $\{0.4, 0.8, 1.0\}$ and set the factor $k$ as 0.5. We also search the start pruning epoch in $\{3, 5, 8, 10, 13\}$.   
For SCE~\cite{wang2019symmetric}, we search the $\alpha$ and $A$ in their proposed symmetric cross entropy loss. The factor $\alpha$ is searched in $\{0.01, 0.10, 0.50, 1.00\}$, and the factor $A$ is searched in $\{-2, -4, -6, -8\}$.
For DAC~\cite{thulasidasan2019combating}, we search the start abstention epoch in $\{3, 5, 8, 10, 15\}$ and set the other hyperparameters as in their paper. 
For APL~\cite{ma2020normalized}, we choose the active loss as normalized cross-entropy loss and the passive loss as reversed cross-entropy loss. We search the loss factor $\alpha$ in $\{1, 10, 100\}$ and $\beta$ in $\{0.1, 1.0, 10\}$. 
For ELR~\cite{liu2020early}, we search the momentum factor $\beta$ in $\{0.5, 0.7, 0.9, 0.99\}$ and the weight factor $\lambda$ in $\{0.05, 0.3, 0.5, 0.7\}$.
For SAT~\cite{huang2020self}, the start epoch is searched in $\{3, 5, 8, 10, 13\}$, and the momentum is searched in $\{0.6, 0.8, 0.9, 0.99\}$. We use the same number of trials in tuning hyperparameters for baselines.

\begin{table*}[!t]
\centering
\caption{Comparing $\sum_{i=1}^L D_i^2$ between layer-wise and constant regularization.}\label{tab:distances}
\resizebox{1.00\columnwidth}{!}{
\begin{tabular}{@{}lccccccc@{}}
\toprule
           & Aircrafts & CUB-200-2011 & Caltech-256 & Stanford-Cars & Stanford-Dogs & Flowers & MIT-Indoor \\ \midrule
Layer-wise & 250.20   &  840.09  & 90.42  & 4150.02 & 20.72  &  252.45  & 57.16 \\
Constant   & 2465.05  &  1225.69 & 195.72 & 2480.08 & 211.78 &  3151.28 & 295.49 \\ \bottomrule
\end{tabular}}
\end{table*}

\subsection{Extended experimental results}\label{sec:extended_exp}

We present additional results. First, we apply our algorithm to another dataset with label noise (of the same type as in Table \ref{tab:noise_res}). Second, we apply our algorithm to sentence classification tasks.

\paragraph{Additional results under label noise.}
We report the test accuracy results of our algorithm on the CUB-200-2011 and Flowers datasets with independent label noise in Table \ref{tab:other_noise_res}. From the table, we show that our algorithm still outperforms previous methods by a significant margin, which aligns with our results of the MIT-Indoor data set in Table \ref{tab:noise_res}. 

\paragraph{Results on ChestX-ray14.} We apply our approach to medical image classification, which is a more distant transfer task (relative to the source data set ImageNet).  
We report the mean AUROC (averaged over predicting all 14 labels) on the ChestX-ray14 data set in Table~\ref{tab:chexnet_results}. With layer-wise regularization, we see $\bm{0.39\%}$ improvement compared to the baseline methods.

\paragraph{Results on ViT.} We apply our approach to fine-tuning pre-trained ViT~\cite{dosovitskiy2020image} from noisy labels, under the same setting as Table~\ref{tab:noise_res}). 
Table~\ref{tab:vit_results} shows the result.
First, we find that our approach outperforms the best regularization methods by $\bm{1.17\%}$, averaged over two settings.
Second, we find that our approach also improves upon self-labeling by $\bm{13.57\%}$ averaged over two settings.
These two results again highlight that both regularization and self-labeling contribute to the final result.
While regularization prevents the model from over-fitting to the random labels, self-labeling injects the belief of the fine-tuned model into the noisy dataset.

\paragraph{Extension to text classification.}
We apply our algorithm in text data domains. We conduct an experiment on sentiment classification tasks using a three-layer feedforward neural network (or multi-layer perceptron). We considered six text classification data sets: SST, MR, CR, MPQA, SUBJ, and TREC. We use SST as the source task and one of the other tasks as the target task. We compared our proposed algorithm with fine-tuning and $\ell_2$-PGM~\cite{gouk2020distance}. In Table \ref{tab:text_res}, we report the results evaluated under a setting with no label noise and a setting with independent label noise (same setting in Table \ref{tab:noise_res}). The results show our proposed algorithm can outperform these baseline methods under both settings.

\begin{figure*}[!t]
    \begin{subfigure}[b]{0.49\textwidth}
    	\centering
	    \includegraphics[width=0.98\textwidth]{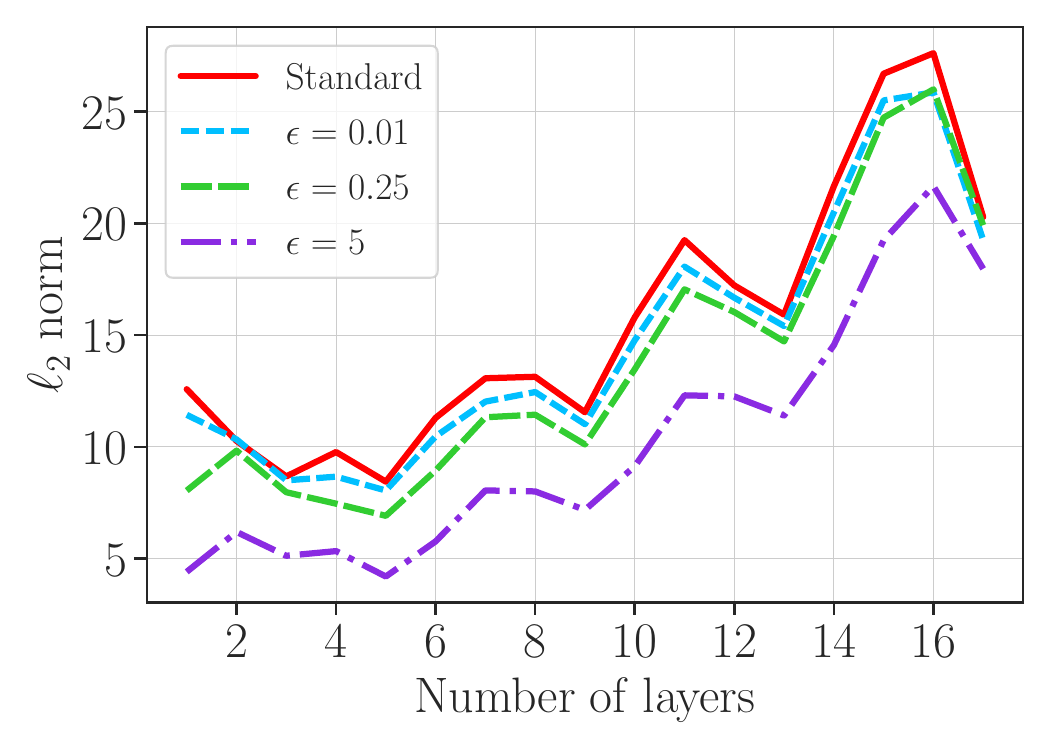}
	    \caption{Norms of each layer in ResNet-18}
    \end{subfigure}
    \begin{subfigure}[b]{0.49\textwidth}
        \centering
        \includegraphics[width=0.98\textwidth]{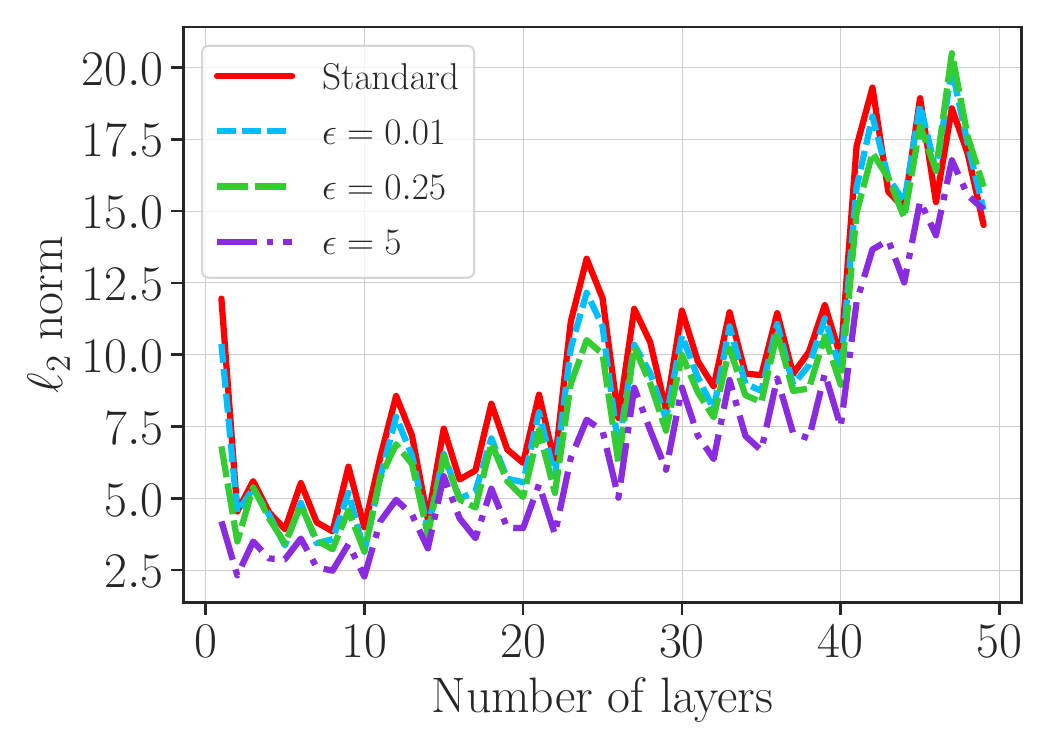}
        \caption{Norms of each layer in ResNet-50}
    \end{subfigure}
    \caption{Frobenius norms of each layer in ResNet-18 pre-trained with standard and adversarial training methods from different noise levels ($\epsilon = 0.01, 0.25, 5$). Higher $\epsilon$ implies lower norms. }\label{fig:robust_model_norm}
\end{figure*}

\subsection{Ablation studies}\label{sec:ablation_study}

We describe several ablation studies. First, we compare the performance between using the Frobenius norm and the MARS norm proposed in~\citet{gouk2020distance}. 
Second, we show the norm values of the adversarial robust pre-trained models in each layer and explain their improvement for fine-tuning on downstream tasks. 
Third, we compare our proposed layer-wise constraints with a constant and discuss their correlation with generalization performance in Theorem \ref{thm_fnn}.
Fourth, we analyze the role of regularization in our proposed self-label-correction method.
Finally, we conduct an ablation study to study the influence of three components in our algorithm. 

\paragraph{Comparing layer-wise and constant regularization: $\sum_{i=1}^L D_i^2$.} To show the bound in Proposition \ref{thm_fnn} is nonvacuous, We compare the value of $\sum_{i=1}^L D_i^2$ from Equation \ref{eq_pac_main} under both layer-wise constraints and constant regularization in Table \ref{tab:distances}. 
The values for every $D_i$ are from the parameters set in Table~\ref{tab:finetuning_res} for these data sets. We can see that layer-wise constraints incur a smaller value of $\sum_{i=1}^L D_i^2$ compared to constant regularization, which correlates with their better generalization performance.

\paragraph{Comparing norms of adversarial robust pre-trained models.}
We empirically observe that adversarially robust pre-trained models have lower Frobenius norms over each layer, as shown in Figure \ref{fig:robust_model_norm}. The lower norms of robust models would induce smaller generation error, which explains their improvement in fine-tuning downstream tasks. 

\begin{table}[!t]
\centering
\caption{Denominator (total number of relabeled data points) and numerator (the number of correctly relabeled data points) for calculating  precision of label-correction in Figure \ref{fig:label_correction}.}\label{tab:numerator}
\begin{tabular}{@{}lllllllll@{}} \toprule
    Number of epoch                & 1  & 4  & 7  & 10 & 13 & 16 & 19 & 21 \\ \midrule
    Denominator with regularization  & 28 & 32 & 32 & 29 & 26 & 24 & 21 & 19 \\
    Numerator with regularization    & 15 & 17 & 16 & 13 & 12 & 10 & 9  & 8  \\
    Denominator without regularization & 56 & 66 & 64 & 64 & 62 & 58 & 59 & 56 \\
    Numerator without regularization   & 34 & 33 & 28 & 26 & 21 & 18 & 17 & 14 \\ \bottomrule
\end{tabular}
\end{table}

\begin{table}[!t]
\centering
\caption{Top-$1$ test accuracy of using different norms for fine-tuning  ResNet-101  pre-trained on the ILSVRC-2012 subset of ImageNet. Results are averaged over 3 random seeds.}\label{tab:compare_norm}
\begin{scriptsize}
\begin{tabular}{lccccccc}
\toprule
              & Aircrafts & CUB-200-2011 & Caltech-256 & Stanford-Cars & Stanford-Dogs & Flowers & MIT-Indoor \\ \midrule
PGM (MARS)      & 74.34$\pm$0.15 & 81.14$\pm$0.16 & 83.28$\pm$0.31 & 89.01$\pm$0.18 & 87.90$\pm$0.10 & 93.45$\pm$0.23 & 78.48$\pm$0.29 \\
PGM ($\ell^2$)  & 74.90$\pm$0.26 & 81.23$\pm$0.32 & 83.25$\pm$0.33 & 88.92$\pm$0.39 & 86.48$\pm$0.28 & 93.23$\pm$0.34 & 77.31$\pm$0.30 \\ 
Ours (MARS)     & 75.10$\pm$0.32 & 82.03$\pm$0.32 & 85.33$\pm$0.33 & 89.29$\pm$0.25 & 90.94$\pm$0.28 & 93.67$\pm$0.34 & 79.40$\pm$0.30 \\
Ours ($\ell^2$) & 75.32$\pm$0.23 & 82.24$\pm$0.21 & 84.90$\pm$0.16 & 89.14$\pm$0.22 & 89.58$\pm$0.13 & 93.82$\pm$0.35 & 79.30$\pm$0.31 \\ \bottomrule
\end{tabular}
\end{scriptsize}
\end{table}

\paragraph{Comparing label-correction precision with and without regularization.}  We observe from Figure \ref{fig:label_correction} that combining self-labeling and regularization is more effective than using just self-labeling. This is evidenced by the gap between the correction with regularization and without regularization in the figure.
In Table \ref{tab:numerator}, we show the denominator of relabeling precision (i.e., the overall number of data points that are relabeled in line 7 of Algorithm \ref{alg:rnc}) in the correction process. We find that the denominator differs between models with and without regularization.
The denominator is significantly smaller with regularization than without, indicating that regularization prevents the model from overfitting the noisy labels.
The denominator decreases only with regularization. Without regularization, the denominator remains relatively unchanged. Additionally, the number of incorrectly labeled data points by the self-labeling is much higher (e.g., 42 vs. 11 at epoch 21).

\paragraph{Comparing $\bm{\ell_2}$ norm and the MARS norm.}
We report the results of using different norms for constraints in our algorithm in Table \ref{tab:compare_norm}. We compare the performance of the Frobenius norm ($\ell^2$), and the MARS norm proposed in~\cite{gouk2020distance}. The table shows that the results of the two norms are similar. In the paper, we focus on the Frobenius norm for our discussion. 

\paragraph{Influence of different components in {\sc RegSL}.}
We study the influence of each component of our algorithm: layer-wise constraint, label correction, and label removal. We remove these components, respectively, and run the same experiments on the MIT-Indoor dataset with different kinds of label noise. Furthermore, we also include a row of results using only layer-wise constraints without self-labeling (containing label correction and removal). 
As shown in Table \ref{tab:noise_ablation}, removing any component from our algorithm can harm its performance. This suggests that incorporating only these components can prevent both overfitting and label memorization in models. In addition, we can see from Table \ref{tab:noise_ablation} that when the noise rate is $20\%$, the self-labeling part (including label correction and removal) is more critical than regularization. When the noise rate is $40\%$ or higher, the label correction part is the most important.

\begin{table}[!t]
\centering
\caption{Removing any component in our algorithm leads to worse performance. Results are on the Indoor dataset with independent label noise. Results are averaged over 3 random seeds.}\label{tab:noise_ablation}
\resizebox{1.00\columnwidth}{!}{
\begin{tabular}{@{}lcccccc@{}}
\toprule
\multirow{2}{*}{Indoor} & \multicolumn{4}{c}{independent noise} & correlated noise \\
                      & 20\%    & 40\%    & 60\%    & 80\%    & 25.18\%          \\ \midrule
{\sc RegSL} (ours)                   & \textbf{72.51$\pm$0.46} & \textbf{68.13$\pm$0.16} & \textbf{57.59$\pm$0.55} & \textbf{34.08$\pm$0.79} & \textbf{70.12$\pm$0.83} \\
w/o regularization     & 71.94$\pm$0.43 & 67.84$\pm$0.38 & 57.24$\pm$0.43 & 33.78$\pm$0.30 & 69.43$\pm$0.36 \\ 
w/o label correction   & 70.92$\pm$0.41 & 59.10$\pm$0.24 & 47.81$\pm$0.35 & 28.42$\pm$0.46 & 69.78$\pm$0.34 \\ 
w/o label removal      & 70.32$\pm$0.65 & 66.57$\pm$0.76 & 55.37$\pm$0.28 & 29.43$\pm$0.88 & 67.96$\pm$0.49 \\ 
w/o self-labeling      & 70.23$\pm$0.25	& 64.40$\pm$0.58 & 54.20$\pm$0.68	& 32.54$\pm$0.43 & 69.05$\pm$0.09 \\
\bottomrule
\end{tabular}}
\end{table}

\end{document}